\documentclass{ecai}
\usepackage{times}
\usepackage{graphicx}
\usepackage{latexsym}
\usepackage{hyperref}
\usepackage[utf8]{inputenc}
\usepackage{url}
\usepackage{booktabs}
\usepackage{amsfonts}
\usepackage{nicefrac}
\usepackage{microtype}
\usepackage{amssymb, amsthm, mathtools}
\usepackage{amsmath}
\usepackage{subfigure}
\usepackage[normalem]{ulem}
\usepackage{graphicx}
\usepackage{physics}
\usepackage{array, multirow}
\usepackage{calc}
\usepackage{tabularx}
\usepackage{tikz}
	\usetikzlibrary{positioning}
	\usetikzlibrary{arrows.meta}
\usepackage{xspace}
\usepackage{pgfplotstable}
\usepackage{enumitem}
\usepackage{numprint}
\pgfplotsset{compat=1.14}
\usepackage{mathdots}
\usepackage{color, colortbl}
\usepackage{wrapfig}
\usepackage{textcomp}
\usepackage{changepage}
\usepackage[english]{babel}
\usepackage{placeins}
\usepackage{changepage}

\ecaisubmission

\theoremstyle{plain}
\newtheorem{thm}{\protect\theoremname}
\theoremstyle{plain}
\newtheorem{prop}[thm]{\protect\propositionname}
\newtheorem{defn}{Definition}
\newtheorem{lem}{Lemma}
\newtheorem{cor}{Corollary}

\providecommand{\propositionname}{Proposition}
\providecommand{\theoremname}{Theorem}

\newcolumntype{C}[1]{>{\centering\arraybackslash}p{#1}}
\newcommand{\C}{{\ensuremath \mathbb C}}

\newcommand{\Cnn}{\C^{n\times n}}
\newcommand{\yt}{\textit{YouTube-8M}\xspace}
\newcommand{\relu}{ReLU\xspace}

\newcommand{\drn}{Deep ReLU network\xspace}

\newcommand{\ci}{\mathbf{i}}

\begin{document}

\title{Understanding and Training Deep Diagonal Circulant Neural Networks}

\author{
Alexandre Araujo\institute{Wavestone,\ Paris, France, email: alexandre.araujo@dauphine.eu}$^{\ ,2}$ \quad 
Benjamin Negrevergne\institute{Université Paris-Dauphine,\ PSL Research University,\ LAMSADE,\ CNRS,\ UMR 7243,\ Paris,\ France} \quad Yann Chevaleyre$^2$ \quad Jamal Atif$^2$}

\maketitle

\begin{abstract}
In this paper, we study deep diagonal circulant neural networks, that is deep neural networks in which weight matrices are the product of diagonal and circulant ones.
Besides making a theoretical analysis of their expressivity, we introduced principled techniques for training these models: we devise an initialization scheme and proposed a smart use of non-linearity functions in order to train deep diagonal circulant networks. 
Furthermore, we show that these networks outperform recently introduced deep networks with other types of structured layers. We conduct a thorough experimental study to compare the performance of deep diagonal circulant networks with state of the art models based on structured matrices and with dense models. We show that our models achieve better accuracy than other structured approaches while required 2x fewer weights as the next best approach. Finally we train deep diagonal circulant networks to build a compact and accurate models on a real world video classification dataset with over 3.8 million training examples. 
\end{abstract}

\section{Introduction}
\label{introduction}

The deep learning revolution has yielded models of increasingly large size. 
In recent years, designing compact and accurate neural networks with a small  number of trainable parameters  has been an active research topic, motivated by practical applications in embedded systems (to reduce memory footprint \cite{43969}), federated and distributed learning (to reduce communication \cite{45648}), derivative-free optimization in reinforcement learning (to simplify the computation of the approximated gradient \cite{47028}). Besides a number of practical applications, it is also an important research question whether or not models really need to be this big or if smaller results can achieve similar accuracy~\cite{ba2014deep}.

Structured matrices are at the very core of most of the work on compact networks. In these models, dense weight matrices are replaced by matrices with a prescribed structure (e.g. low rank matrices, Toeplitz matrices, circulant matrices, LDR, etc.). Despite substantial efforts  (e.g. \cite{cheng,moczulski2015acdc}), the performance of compact models is still far from achieving an acceptable accuracy motivating their use in real-world scenarios. This raises several questions about the effectiveness of such models and about our ability to train them. In particular two main questions call for investigation:
\begin{itemize}
    \item[]{\bf Q1\ }{\em  How to efficiently train deep neural networks with a large number of structured layers?}
    \item[]{\bf Q2\ }{\em What is the expressive power of structured layers compared to dense layers?}
\end{itemize}

In this paper, we provide principled answers to these questions for the particular case of deep neural networks based on diagonal and circulant matrices (a.k.a. Diagonal-circulant networks or DCNNs). 

The idea of using diagonal and circulant matrices together comes from a series of results in linear algebra by Muller et al. \cite{muller1998algorithmic} and Huhtanen et al. \cite{Huhtanen2015}. The most recent result from Huhtanen et al.  \cite{Huhtanen2015} demonstrates that any matrix $A$ in $\Cnn$ can be decomposed into the product of $2n-1$ alternating diagonal and circulant matrices.
The diagonal-circulant decomposition inspired Moczulski et al.  \cite{moczulski2015acdc} to design the {\em AFDF} structured layer, which is the building block of DCNNs. However, they were not able to train deep neural networks based on AFDF. 

To answer {\bf Q1}, we first describe a theoretically sound initialization procedure for DCNN which allows the signal to propagate through the network without vanishing or exploding. Furthermore, we provide a number of empirical insights to explain the behaviour of DCNNs, and show the impact of the number of the non-linearities in the network on the convergence rate and the accuracy of the network. 
By combining all these insights, we are able (for the first time) to train large and deep DCNNs. We demonstrate the good performance of DCNNs on a large scale application (the \yt video classification problem) and obtain very competitive accuracy. 

To answer {\bf Q2}, we propose an analysis of the expressivity of DCNNs by extending the results by Huhtanen et al. \cite{Huhtanen2015}. We introduce a new bound on the number of  diagonal-circulant required to approximate a matrix that depends on its rank. Building on this result, we demonstrate that a DCNN with bounded width and small depth can approximate any dense networks with ReLU activations. 

\paragraph{Outline of the paper:} We present in Section~\ref{related_work} the related work on structured neural networks and several compression techniques. Section~\ref{section:circulant} introduces circulant matrices, our new result extending the one from Huhtanen et al. \cite{Huhtanen2015}. Section~\ref{sec-dcnn-th} proposes an theoretical analysis on the expressivity on DCNNs. Section~\ref{section:training} describes two efficient techniques for training  deep diagonal circulant neural networks. Finally, Section~\ref{section:exp} presents extensive experiments to compare the performance of deep diagonal circulant neural networks in different settings w.r.t. other state of the art approaches. Section~\ref{section:conclusion} provides a discussion and concluding remarks.

\section{Related Work}
\label{related_work}

Structured matrices exhibit a number of good properties which have been exploited by deep learning practitioners, mainly to compress large neural networks architectures into smaller ones. For example Hinrichs et al. \cite{hinrichs2011johnson} have demonstrated that a single circulant matrix can be used to approximate the Johson-Lindenstrauss transform, often used in machine learning to perform dimensionality reduction. Building upon this result, Cheng et al. \cite{cheng} proposed to replace the weight matrix of a fully connected layer by a circulant matrix effectively replacing the complex transform modeled by the fully connected layer by a simple dimensionality reduction. Despite the reduction of expressivity, the resulting network demonstrated good accuracy using only a fraction of its original size (90\% reduction).

\textbf{Comparison with ACDC.}
Moczulski et al. \cite{moczulski2015acdc} have introduced two {\em Structured Efficient Linear Layers} (SELL) called AFDF and ACDC. The AFDF structured layer benefits from the theoretical results introduced by Huhtanen et al. \cite{Huhtanen2015} and can be seen the building block of DCNNs.
However, Moczulski et al. \cite{moczulski2015acdc} only experiment using  ACDC, a different type of layer that does not involve circulant matrices. As far as we can tell, the theoretical guarantees available for the AFDF layer do not apply on the ACDC layer since the cosine transform does not diagonalize circulant matrices \cite{sanchez1995diagonalizing}. Another possible limit of the ACDC paper is that they only train large neural networks involving ACDC layers combined with many other expressive layers. Although the resulting network demonstrates good accuracy, it is difficult the characterize the true contribution of the ACDC layers in this setting. 

\textbf{Comparison with Low displacement rank structures.} More recently, Thomas et al. \cite{Thomas_NIPS2018_8119} have generalized these works by proposing neural networks with low-displacement rank matrices (LDR), that are structured matrices encompassing a large family of structured matrices, including Toeplitz-like, Vandermonde-like, Cauchy-like and more notably DCNNs. To obtain this result, LDR represents a structured matrix using two displacement operators and a low-rank residual. Despite being elegant and general, we found that the LDR framework suffers from several limits which are inherent to its generality, and makes it difficult to use in the context of large and deep neural networks. First, the training procedure for learning LDR matrices is highly involved and implies many complex mathematical objects such as Krylov matrices. Then, as acknowledged by the authors, the number of parameters required to represent a given structured matrix (e.g. a Toeplitz matrix) in practice is unnecessarily high (higher than required in theory). 

\textbf{Other compression techniques.} Besides structured matrices, a variety of techniques have been proposed to build more compact deep learning models. These include {\em model distillation}~\cite{44873}, Tensor Train~\cite{novikov2015tensorizing}, Low-rank decomposition~\cite{NIPS2013_5025}, to mention a few. However, Circulant networks show good performances in several contexts (the interested reader can refer to the results reported by Moczulski et al. \cite{moczulski2015acdc} and Thomas et al. \cite{Thomas_NIPS2018_8119}).

\section{A primer on circulant matrices and a new result}
\label{section:circulant}

An n-by-n circulant matrix $C$ is a matrix where each row is a cyclic right shift of the previous one as illustrated below.

{\small \[
    C = circ(c) =\left[\begin{array}{ccccc}
    c_{0} & c_{n-1} & c_{n-2} & \dots & c_{1} \\
    c_{1} & c_{0} & c_{n-1} & & c_{2} \\
    c_{2} & c_{1} & c_{0}& & c_{3} \\
    \vdots & & & \ddots & \vdots \\
    c_{n-1} & c_{n-2} & c_{n-3} & & \phantom{0}c_{0}\phantom{0}
    \end{array}\right]
\]}

\noindent
Circulant matrices exhibit several interesting properties from the perspective of numerical computations. Most importantly, any $n$-by-$n$ circulant matrix $C$ can be represented using only $n$ coefficients instead of the $n^2$ coefficients required to represent classical unstructured matrices. In addition, the matrix-vector product is simplified from $O(n^2)$ to $O(n\ log(n))$ using the  convolution theorem.

As we will show in this paper, circulant matrices also have a strong expressive power. So far, we know that a single circulant matrix can be used to represent a variety of important linear transforms such as random projections~\cite{hinrichs2011johnson}. 
When they are combined with diagonal matrices, they can also be used as building blocks to represent any linear transform~\cite{schmid2000decomposing, Huhtanen2015} with an arbitrary precision. Huhtanen et al. \cite{Huhtanen2015} were able to bound the number of factors that is required to approximate any matrix $A$ with arbitrary precision.

\paragraph{Relation between diagonal circulant matrices and low rank matrices}
We recall this result in  Theorem~\ref{thm:huhtanen} as it is the starting point of our theoretical analysis (note that in the rest of the paper, $\left\Vert\ \cdot\ \right\Vert $ denotes the $\ell_{2}$ norm when applied to vectors, and the operator norm when applied to matrices).

\begin{thm}
(Reformulation from Huhtanen et al. \cite{Huhtanen2015})
\label{thm:huhtanen}
For every matrix $A\in\mathbb{C}^{n\times n}$, for any $\epsilon > 0$, there exists a sequence of matrices $B_1 \ldots B_{2n-1}$ where $B_{i}$ is a circulant matrix if $i$ is odd, and a diagonal matrix otherwise, such that $\left\Vert B_{1}B_{2}\ldots B_{2n-1}-A \right\Vert < \epsilon$.
\end{thm}

Unfortunately, this theorem is of little use to understand the expressive power of diagonal-circulant matrices when they are used in deep neural networks. This is because: 1) the bound only depends on the dimension of the matrix $A$, not on the matrix itself, 2) the theorem does not provide any insights regarding the expressive power of $m$ diagonal-circulant factors when $m$ is much lower than $2n - 1$ as it is the case in most practical scenarios we consider in this paper. 

In the following theorem, we enhance the result by Huhtanen et al. \cite{Huhtanen2015} by expressing the number of factors required to approximate $A$, {\em as a function of the rank of $A$}. This is useful when one deals with low-rank matrices, which is common in machine learning problems. 

\begin{thm} \footnote{All proofs are in the arxiv version of the paper. \\ \url{https://arxiv.org/abs/1901.10255}} (Rank-based circulant decomposition)
\label{prop:rank-decomposition}Let $A\in\mathbb{C}^{n\times n}$ be
a matrix of rank at most $k$. Assume that $n$ can be divided by $k$. For
any $\epsilon>0$, there exists a sequence of $4k+1$ matrices $B_{1},\ldots,B_{4k+1},$ where $B_{i}$ is a circulant matrix if $i$ is odd, and a diagonal matrix otherwise, such that $\Vert B_1B_2\ldots B_{4k+1} - A\Vert < \epsilon$
\end{thm}

A direct consequence of Theorem~\ref{prop:rank-decomposition}, is that if the number of diagonal-circulant factors is set to a value $K$, we can represent all linear transform $A$ whose rank is $\frac{K - 1}{4}$.

Compared to \cite{Huhtanen2015}, this result shows that structured matrices with fewer than $2n$ diagonal-circulant matrices (as it is the case in practice) can still represent a large class of matrices. As we will show in the following section, this result will be useful to analyze the expressivity of neural networks based on diagonal and circulant matrices.

 \section{Analysis of Diagonal Circulant Neural
Networks (DCNNs)}
\label{sec-dcnn-th}

Zhao et al. \cite{pmlr-v70-zhao17b} have shown that circulant networks with 2 layers and unbounded width are universal approximators. However, results on unbounded networks offer weak guarantees and two important questions have remained open until now: 1) {\em Can we approximate any function with a bounded-width circulant networks?} 2) {\em What function can we approximate with a circulant network that has a bounded width and a small depth?} 
We answer these two questions in this section.

First, we introduce some necessary definitions regarding neural networks and we provide a theoretical analysis of their approximation capabilities.

\begin{defn}[\drn]\label{drn}
Given $L$ weight matrices $W = (W_1, \ldots, W_L)$ with $W_i \in \mathbb C^{n\times n}$ and  $L$ bias vectors $b = (b_1, \ldots, b_L)$  with  $b_i \in \mathbb C^n$, a {\em deep \relu network} is a function $f_{W_L, b_L} : \mathbb C^n \rightarrow \mathbb C^n$ such that $f_{W, b}(x) =  (f_{W_L, b_L} \circ \ldots \circ f_{W_1, b_1})(x)$ where $f_{W_i, b_i}(x) = \phi(W_i x + b_i)$ and $\phi(.)$ is a \relu non-linearity
\footnote{Because our networks deal with complex numbers, we use an extension of the \relu function to the complex domain. The most straightforward extension defined in \cite{DBLP:conf/iclr/TrabelsiBZSSSMR18} is as follows: $\mathrm{\relu}(z)=\mathrm{ReLU}\left(\mathfrak{R}(z)\right)+i\mathrm{ReLU}\left(\mathfrak{I}(z)\right)$, where $\mathfrak{R}$ and $\mathfrak{I}$ refer to the real and imaginary parts of $z$.}
In the rest of this paper, we call $L$ and $n$ respectively the depth and the width of the network. Moreover, we call {\em total rank $k$}, the sum of the ranks of the matrices $W_{1}\ldots W_{L}$. i.e. $k = \sum_{i=1}^L rank(W_i)$.
\end{defn}

\noindent
We also need to introduce DCNNs, similarly to Moczulski et al.  \cite{moczulski2015acdc}.

\begin{defn}[Diagonal Circulant Neural Networks]\label{def:DCNN}
Given $L$ diagonal matrices $D = (D_1, \ldots, D_L)$ with $D_i \in \mathbb C^{n\times n}$, $L$ circulant matrices $C = (C_1, \ldots, C_L)$ with $C_i \in \mathbb C^{n\times n}$ and $L$ bias vectors $b = (b_1, \ldots, b_L)$ with  $b_i \in \mathbb C^n$, a {\em Diagonal Circulant Neural Networks} (DCNN) 
is a function  $f_{W_L, b_L} : \mathbb C^n \rightarrow \mathbb C^n$ such that $f_{D,C,b}(x) = (f_{D_L, C_L, b_L} \circ \ldots \circ f_{D_1, C_1, b_1})(x)$ where $f_{D_i, C_i, b_i}(x) = \phi_i(D_i C_i x + b_i)$ and where $\phi_i(.)$ is a \relu non-linearity or the identity function.
\end{defn}

\noindent
We can now show that bounded-width DCNNs can approximate any Deep ReLU Network, and as a corollary, that they are universal approximators.

\begin{lem}\label{mainth_}
Let $\mathcal{N}$ be a deep ReLU network of width $n$ and depth $L$, and let $\mathcal{X} \subset \mathbb{C}^{n}$ be a bounded set. For any $\epsilon>0$, there exists a DCNN $\mathcal{N}'$ of width $n$ and of depth $(2n-1)L$ such that $\Vert \mathcal{N}(x)-\mathcal{N}'(x) \Vert < \epsilon$ for all $x \in \mathcal{X}$.
\end{lem}

\noindent
We can now state the universal approximation corollary:

\begin{cor}\label{cor:universal}
Bounded width DCNNs are universal approximators in the following sense: for any continuous function $f:[0,1]^{n}\rightarrow\mathbb{R}_+$ of bounded supremum norm,
for any $\epsilon>0$, there exists a DCNN
$\mathcal{N}_{\epsilon}$ of width $n+3$ such that $\forall x\in[0,1]^{n+3}$, $\left|f(x_{1}\ldots x_{n})-\left(\mathcal{N}_{\epsilon}\left(x\right)\right)_{1}\right|<\epsilon$, where $\left(\cdot\right)_{i}$ represents the $i^{th}$ component of a vector.
\end{cor}

\noindent
This is a first result, however $(2n+5)L$ is not a small depth (in our experiments, $n$ can be over 300~000), and a number of work provided empirical evidences that DCNN with small depth can offer good performances (e.g. \cite{anca2018eccv,cheng}). To improve our result, we introduce our main theorem which studies the approximation properties of these small depth networks.

\begin{thm}(Rank-based expressive power of DCNNs)
\label{prop:low_rank_nn}
\noindent Let $\mathcal{N}$ be a deep ReLU network of width $n$, depth $L$ and a total rank $k$ and assume $n$ is a power of $2$. Let $\mathcal{X} \subset \mathbb{C}^{n}$ be a bounded set. Then, for any $\epsilon>0$, there exists a DCNN with ReLU activation $\mathcal{N}'$ of width $n$ such that $\left\Vert \mathcal{N}(x)-\mathcal{N}'(x)\right\Vert <\epsilon$ for all $x\in\mathcal{X}$ and the depth of $\mathcal{N}'$ is bounded by $9k$.\end{thm}

\noindent
Remark that in the theorem, we require that $n$ is a power of $2$. We conjecture that the result still holds even without this condition. This result refines Lemma~\ref{mainth_}, and answer our second question: a DCNN of bounded width and small depth can approximate a Deep ReLU network of low  total rank. Note that the converse is not true: because $n$-by-$n$ circulant matrix can be of rank $n$, approximating a DCNN of depth $1$ can require a deep ReLU network of total rank equals to $n$.

\paragraph{Expressivity of DCNNs}

For the sake of clarity, we highlight the significance of these results with the two following properties.

\textbf{Properties. }
Given an arbitrary fixed integer $n$, let $\text{\ensuremath{\mathcal{R}}}_{k}$ be the set of all
functions $f:\mathbb{R}^{n}\rightarrow\mathbb{R}^{n}$ representable
by a deep ReLU network of total rank at most $k$ and  let $\mathcal{C}_{l}$ the set of all functions $f:\mathbb{R}^{n}\rightarrow\mathbb{R}^{n}$
representable by deep diagonal-circulant networks of depth at most
$l$, then:
\begin{align}
\label{prop_eq1}\forall k,\exists l \, & \quad \mathcal{R}_{k}\subsetneq\mathcal{C}_{l} \\
\label{prop_eq2}\forall l,\nexists k\, & \quad \mathcal{C}_{l}\subseteq\mathcal{R}_{k}
\end{align}

\noindent
We illustrate the meaning of this properties 
using Figure~\ref{fig:circfig}. As we can see, the set $\mathcal R_{k}$ of all the functions representable by a deep \relu network of total rank $k$ is strictly included in the set $\mathcal C_{9k}$ of all DCNN of depth $9k$ (as by Theorem~\ref{prop:low_rank_nn}).

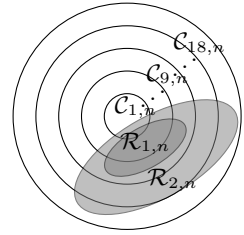
\begin{figure}[th]
    \begin{center}
    \tikzset{%
  >={Latex[width=2mm,length=2mm]},
            base/.style = {rectangle, draw=black, text centered, font=\sffamily},
           other/.style = {base, fill=none,  minimum width=1.7cm, minimum height=0.7cm},
         ellipse/.style = {base}
}
\begin{tikzpicture}[every node/.style={fill=white, font=\sffamily}, align=center,scale=0.6]

    \draw (0,0) circle (2.5cm);
    \draw (0,0) circle (2.0cm);
    \draw (0,0) circle (1.5cm);
    \draw (0,0) circle (1.0cm);
    \draw (0,0) circle (0.5cm);
    
    \draw[ellipse, rotate=30, fill=gray, opacity=0.5] (0.1, -1.1) ellipse (2.0cm and 0.9cm);
    \draw[ellipse, rotate=30, fill=gray, opacity=0.5] (0.0, -0.8) ellipse (1.0cm and 0.45cm);

    \node[other, draw=none] at (0.20, 0.20) {$\mathcal{C}_{1,n}$};
    \node[other, draw=none] at (0.55, 0.55) {$\iddots$};
    \node[other, draw=none] at (0.90, 0.90) {$\mathcal{C}_{9,n}$};
    \node[other, draw=none] at (1.25, 1.25) {$\iddots$};
    \node[other, draw=none] at (1.60, 1.60) {$\mathcal{C}_{18,n}$};
    
    \node[other, draw=none] at (0.4, -0.6) {$\mathcal{R}_{1,n}$};
    \node[other, draw=none] at (1.0, -1.4) {$\mathcal{R}_{2,n}$};

\end{tikzpicture}
    \end{center}
    \caption{Illustration of Properties (1) and (2).
    }
    \label{fig:circfig}
\end{figure}

\noindent
These properties are interesting for many reasons. 
First, Property (\ref{prop_eq2}) shows that diagonal-circulant networks are \emph{strictly more expressive} than networks with low total rank. 
Second and most importantly, in standard deep neural networks, it is known that the most of the singular values are close to zero (see e.g. \cite{Sedghi18iclr,Arora19neurips}). Property (\ref{prop_eq1}) shows that these networks can efficiently be approximated by  diagonal-circulant networks. Finally, several publications have shown that neural networks can be trained explicitly to have low-rank weight matrices \cite{chong18eccv, goyal19}. This opens the possibility of learning compact and accurate diagonal-circulant networks.

\section{How to train very deep DCNNs}
\label{section:training}

\begin{figure*}[ht]
\centering
\subfigure[]{
    \centering
    \begin{tikzpicture}[scale=0.5]
\begin{axis}[
    legend cell align={left},
    xlabel={\large \#layers},
    ylabel={Test Accuracy},
    xmin=2, xmax=40,
    legend pos=south west,
    ymajorgrids=true,
    grid style=dashed,
	]
  \addplot[mark=square, color=blue, line width=0.4mm] table [y=accuracy, x=layers]{data/cifar10/factor/1_factor.dat};
  \addplot[mark=triangle, color=brown, line width=0.4mm] table [y=accuracy, x=layers]{data/cifar10/factor/2_factor.dat};
  \addplot[mark=o, color=gray, line width=0.4mm] table [y=accuracy, x=layers]{data/cifar10/factor/3_factor.dat};
    \legend{
      ReLU(DC),
      ReLU(DCDC), 
      ReLU(DCDCDC), 
     }
\end{axis}
\end{tikzpicture}
    \label{fig:cifar10_factor}
    }\hspace{2cm}
\subfigure[]{
    \centering
    \begin{tikzpicture}[scale=0.5]
\begin{axis}[
    legend cell align={left},
    xlabel={\large \#layers},
    ylabel={Test Accuracy},
    xmin=2, xmax=40,
    legend pos=south west,
    ymajorgrids=true,
    grid style=dashed,
	]
  \addplot[mark=square, color=blue, line width=0.4mm] table [y=accuracy, x=layers]{data/cifar10/leaky_relu/slope_0.2.dat};
  \addplot[mark=triangle, color=brown, line width=0.4mm] table [y=accuracy, x=layers]{data/cifar10/leaky_relu/slope_0.3.dat};
  \addplot[mark=o, color=gray, line width=0.4mm] table [y=accuracy, x=layers]{data/cifar10/leaky_relu/slope_0.5.dat};
    \legend{
      Leaky ReLU 0.2,
      Leaky ReLU 0.3,
      Leaky ReLU 0.5,
     }
\end{axis}
\end{tikzpicture}
    \label{fig:cifar10_leaky_relu}
    }
    \caption{Experiments on training DCNNs and other structured neural networks on CIFAR-10. Figure~\ref{fig:cifar10_factor}: impact of increasing the number of ReLU activations in a DCNN. Deep DCNNs with fewer ReLUs are easier to train.
    Figure~\ref{fig:cifar10_leaky_relu}: impact of increasing the slope of a Leaky-ReLU in DCNNs. Deep DCNNs with a larger slope are easier to train.}
\end{figure*}
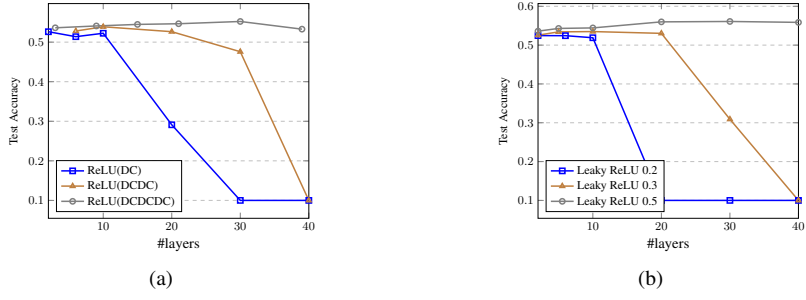

Training DCNNs has revealed to be a challenging problem. We devise two techniques to facilitate the training of deep DCNNs. First, we propose an initialization procedure which guarantee the signal is propagated across the network without vanishing nor exploding. Secondly, we study the behavior of DCNNs with different non-linearity functions and determine the best parameters for different settings. 

\paragraph{Initialization scheme} The following initialization procedure which is a variant of Xavier initialization. First, for each circulant matrix $C=circ(c_{1}\ldots c_{n})$, each $c_{i}$ is randomly drawn from $\mathcal{N}\left(0,\sigma^{2}\right)$, with $\sigma=\sqrt{\frac{2}{n}}$. Next, for each diagonal matrix $D=diag(d_{1}\ldots d_{n})$, each $d_{i}$ is drawn randomly and uniformly from $\{-1,1\}$ for all $i$. Finally, all biases in the network are randomly drawn from $\mathcal{N}\left(0,\sigma'^{2}\right)$, for some small value of $\sigma'$.
The following proposition states that the covariance matrix at the output of any layer in a DCNN, independent of the depth, is constant.

\begin{prop}
\label{prop:initialization}Let $\mathcal{N}$ be a DCNN of depth $L$ initialized according to our procedure, with $\sigma'=0$. Assume that all layers $1$ to $L-1$ have ReLU activation functions, and that the last layer has the identity activation function. Then, for any $x\in\mathbb{R}^{n}$, the covariance matrix of $\mathcal{N}(x)$ is $\frac{2.Id}{n}\left\Vert x\right\Vert _{2}^{2}$.
Moreover, note that this covariance does not depend on the depth of
the network.
\end{prop}



\begin{proof} \emph{(Proposition \ref{prop:initialization})}
Let $\mathcal{N}=\ensuremath{f_{D_{L},C_{L}}\circ\ldots\circ f_{D_{1},C_{1}}}$
be a $L$ layer DCNN. All matrices are initialized as described
in the statement of the proposition. Let $y=D_{1}C_{1}x$. Lemma \ref{lem:covariance}
shows that $cov(y_{i},y_{i'})=0$ for $i\neq i'$ and $var(y_{i})=\frac{2}{n}\left\Vert x\right\Vert _{2}^{2}$.
For any $j\le L$, define $z^{j}=f_{D_{j},C_{j}}\circ\ldots\circ f_{D_{1},C_{1}}(x)$.
By a recursive application of lemma \ref{lem:covariance}, we get
that then $cov(z_{i}^{j},z_{i'}^{j})=0$ and $var(z_{i}^{j})=\frac{2}{n}\left\Vert x\right\Vert _{2}^{2}$.
\end{proof}
\begin{lem}
\label{lem:covariance}Let $c_{1}\ldots c_{n},d_{1}\ldots d_{n},b_{1}\ldots b_{n}$
be random variables in $\mathbb{R}$ such that $c_{i}\sim\mathcal{N}(0,\sigma^{2})$, $b_{i}\sim\mathcal{N}(0,\sigma'^{2})$
and $d_{i}\sim\{-1,1\}$ uniformly. Define $C=circ(c_{1}\ldots c_{n})$
and $D=diag(d_{1}\ldots d_{n})$. Define $y=DCu$ and $z=CDu$ for
some vector $u$ in $\mathbb{R}^{n}$. Also define $\bar{y}=y+b$ and $\bar{z}=z+b$. Then, for all $i$, the p.d.f.
of $y_{i}$, $\bar{y}_{i}$, $z_{i}$ and $\bar{z}_{i}$ are symmetric. Also:
\begin{itemize}
\item Assume $u_{1}\ldots u_{n}$ is fixed. Then, we have for $i\neq i':$
\begin{align*}
cov(y_{i},y_{i'}) & =cov(z_{i},z_{i'}) =cov(\bar{y}_{i},\bar{y}_{i'}) =cov(\bar{z}_{i},\bar{z}_{i'})=0\\
var(y_{i}) & =var(z_{i})=\sum_{j}u_{j}^{2}\sigma^{2} \\
var(\bar{y}_{i}) & =var(\bar{z}_{i})=\sigma'^2+\sum_{j}u_{j}^{2}\sigma^{2}
\end{align*}
\item Let $x_{1}\ldots x_{n}$ be random variables in $\mathbb{R}$ such
that the p.d.f. of $x_{i}$ is symmetric for all $i$, and let $u_{i}=ReLU(x_{i})$.
We have for $i\neq i':$
\begin{align*}
cov(y_{i},y_{i'}) & =cov(z_{i},z_{i'})  =cov(\bar{y}_{i},\bar{y}_{i'})  =cov(\bar{z}_{i},\bar{z}_{i'})=0\\
var(y_{i}) & =var(z_{i})=\frac{1}{2}\sum_{j}var(x_{i}).\sigma^{2} \\
var(\bar{y}_{i}) & =var(\bar{z}_{i})=\sigma'^2+\frac{1}{2}\sum_{j}var(x_{i}).\sigma^{2}
\end{align*}
\end{itemize}
\end{lem}

\begin{proof} \emph{(Lemma \ref{lem:covariance})}
By an abuse of notation, we write $c_{0}=c_{n},c_{-1}=c_{n-1}$ and
so on. First, note that: $y_{i}=\sum_{j=1}^{n}c_{j-i}u_{j}d_{j}$
and $z_{i}=\sum_{j=1}^{n}c_{j-i}u_{j}d_{i}$. Observe that each term
$c_{j-i}u_{j}d_{j}$ and $c_{j-i}u_{j}d_{i}$ have symmetric p.d.f.
because of $d_{i}$ and $d_{j}$. Thus, $y_{i}$ and $z_{i}$ have
symmetric p.d.f. Now let us compute the covariance.

\begin{align*}
    cov(y_{i},y_{i'}) &= \sum_{j,j'=1}^{n}cov\left(c_{j-i}u_{j}d_{j},c_{j'-i'}u_{j'}d_{j'}\right) \\
    \begin{split}
        &= \sum_{j,j'=1}^{n}\mathbb{E}\left[c_{j-i}u_{j}d_{j}c_{j'-i'}u_{j'}d_{j'}\right] \\ &\quad-\mathbb{E}\left[c_{j-i}u_{j}d_{j}\right]\mathbb{E}\left[c_{j'-i'}u_{j'}d_{j'}\right]
    \end{split}
\end{align*}

Observe that $\mathbb{E}\left[c_{j-i}u_{j}d_{j}\right]=\mathbb{E}\left[c_{j-i}u_{j}\right]\mathbb{E}\left[d_{j}\right]=0$
because $d_{j}$ is independent from $c_{j-i}u_{j}$. Also, observe
that if $j\neq j'$ then $\mathbb{E}\left[d_{j}d_{j'}\right]=0$ and
thus $\mathbb{E}\left[c_{j-i}u_{j}d_{j}c_{j'-i'}u_{j'}d_{j'}\right]=\mathbb{E}\left[d_{j}d_{j'}\right]\mathbb{E}\left[c_{j-i}u_{j}c_{j'-i'}u_{j'}\right]=0$.
Thus, the only non null terms are those for which $j=j'$. We get:
\begin{align*}
cov(y_{i},y_{i'}) & =\sum_{j=1}^{n}\mathbb{E}\left[c_{j-i}u_{j}d_{j}c_{j-i'}u_{j}d_{j}\right]\\
 & =\sum_{j=1}^{n}\mathbb{E}\left[c_{j-i}c_{j-i'}u_{j}^{2}\right]
\end{align*}
Assume $u$ is a fixed vector. Then, $var(y_{i})=\sum_{j=1}^{n}u_{j}^{2}\sigma^{2}$
and $cov(y_{i},y_{i'})=0$ for $i\neq i'$ because $c_{j-i}$ is independent
from $c_{j-i'}$.
Now assume that $u_{j}=ReLU(x_{j})$ where $x_{j}$ is a r.v. Clearly,
$u_{j}^{2}$ is independent from $c_{j-i}$ and $c_{j-i'}$. Thus:
\begin{align*}
cov(y_{i},y_{i'}) & =\sum_{j=1}^{n}\mathbb{E}\left[c_{j-i}c_{j-i'}\right]\mathbb{E}\left[u_{j}^{2}\right]
\end{align*}
For $i\neq i'$, then $c_{j-i}$ and $c_{j-i'}$ are independent,
and thus $\mathbb{E}\left[c_{j-i}c_{j-i'}\right]=\mathbb{E}\left[c_{j-i}\right]\mathbb{E}\left[c_{j-i'}\right]=0$.
Therefore, $cov(y_{i},y_{i'})=0$ if $i\neq i'$. Let us compute the
variance. We get $var(y_{i})=\sum_{j=1}^{n}var(c_{j-i}).\mathbb{E}\left[u_{j}^{2}\right]$.
Because the p.d.f. of $x_{j}$ is symmetric, $\mathbb{E}\left[x_{j}^{2}\right]=2\mathbb{E}\left[u_{j}^{2}\right]$
and $\mathbb{E}\left[x_{j}\right]=0$. Thus, $var(y_{i})=\frac{1}{2}\sum_{j=1}^{n}var(c_{j-i}).\mathbb{E}\left[x_{j}^{2}\right]=\frac{1}{2}\sum_{j=1}^{n}var(c_{j-i}).var(x_{j})$.

Finally, note that $cov(\bar{y}_{i},\bar{y}_{i'})=cov(y_{i},y_{i'})+cov(b_{i},b_{i'})$. This yields the covariances of $\bar{y}$.

To derive $cov(z_{i},z_{i'})$ and $cov(\bar{z}_{i},\bar{z}_{i'})$ , the required calculus
is nearly identical. We let the reader check by himself/herself.
\end{proof}


\paragraph{Non-linearity function}

We empirically found that reducing the number of non-linearities in the networks simplifies the training of deep neural networks. To support this claim, we conduct a series of experiments on various DCNNs with a varying number of ReLU activations (to reduce the number of non-linearities, we replace some ReLU activations with the identity function). In a second experiment, we replace the ReLU activations with Leaky-ReLU activations and vary the slope of the Leaky ReLU (a higher slope means an activation function that is closer to a linear function). The results of this experiment are presented in Figure~\ref{fig:cifar10_factor} and \ref{fig:cifar10_leaky_relu}. In \ref{fig:cifar10_factor}, ``ReLU(DC)'' means that we interleave on ReLU activation functions between every diagonal-circulant matrix, whereas ReLU(DCDC) means we interleave a ReLU activation every other block etc.  In both Figure~\ref{fig:cifar10_factor} and  Figure~\ref{fig:cifar10_leaky_relu}, we observe that reducing the non-linearity of the networks can be used to train deeper networks. This is an interesting result, since  we can use this technique to adjust the number of parameters in the network, without facing training difficulties. We obtain a maximum accuracy of 0.56 with one ReLU every three layers and leaky-ReLUs with a slope of 0.5. We hence rely on this setting in the experimental section. 

\section{Empirical evaluation}
\label{section:exp}

This experimental section aims at answering the following questions:
\begin{itemize}
    \item[] {\bf Q6.1} -- How do DCNNs compare to other approaches such as ACDC, LDR or other structured approaches?
    \item[] {\bf Q6.2} -- How do DCNNs compare to other compression based techniques?
    \item[] {\bf Q6.3} -- How do DCNNs perform in the context of large scale real-world machine learning applications?  
\end{itemize}

\subsection{Comparison with other structured approaches (Q6.1)}

\begin{figure*}[ht]
\centering
\subfigure[]{
    \includegraphics[scale=0.35]{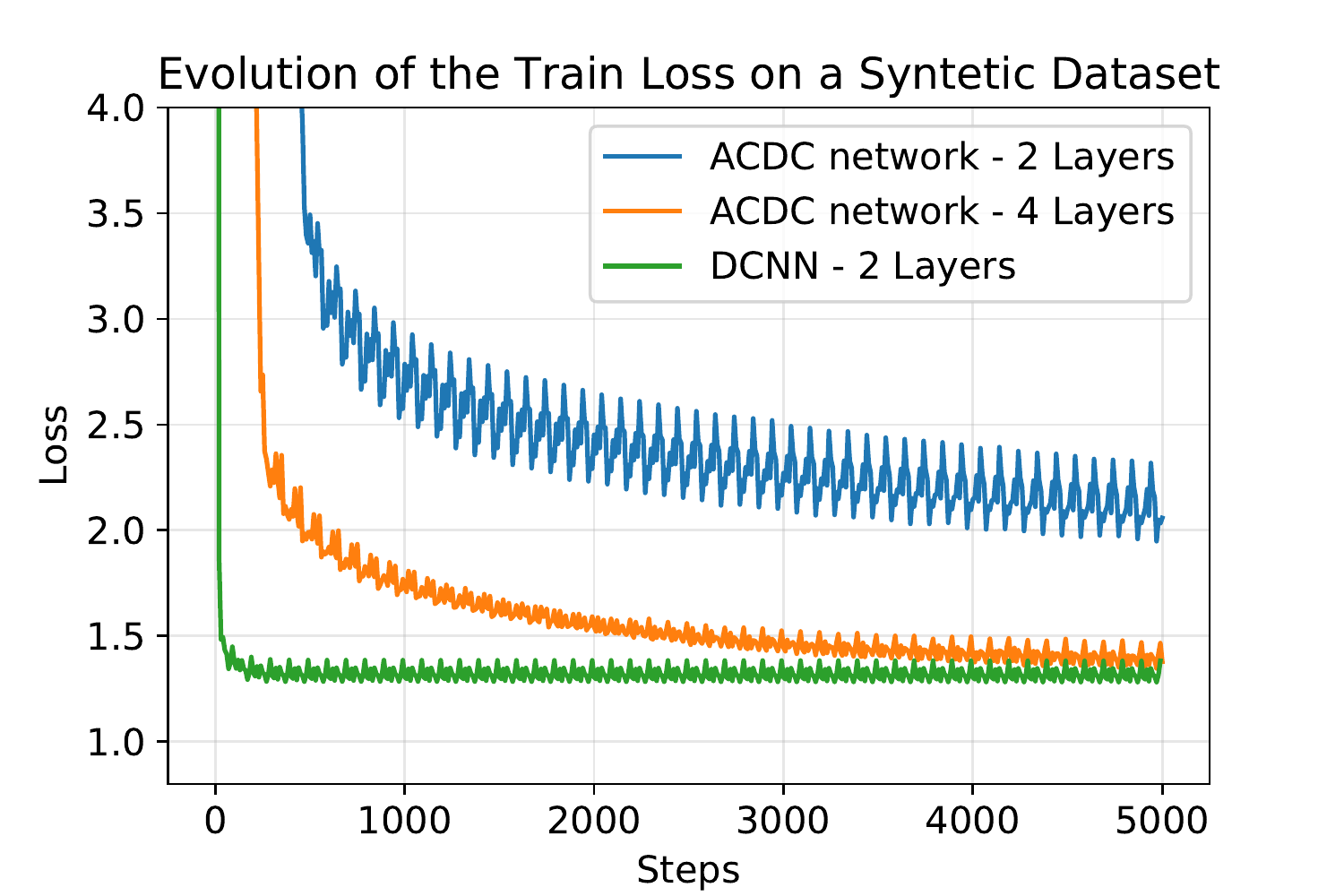}
    \label{fig:adcd_regression}
    }
\subfigure[]{
    \includegraphics[scale=0.35]{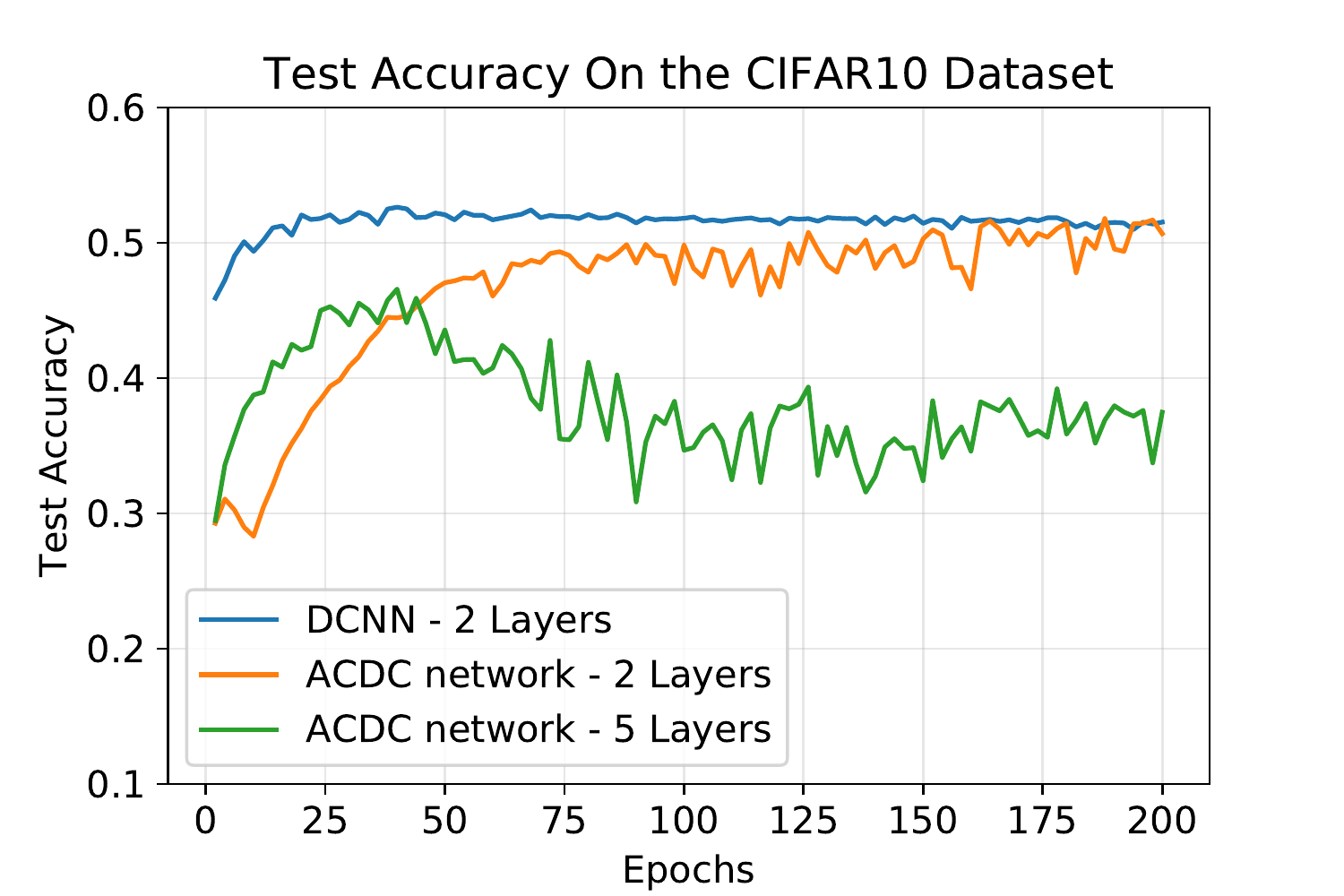}
    \label{fig:acdc_cifar10}
    }
\caption{Comparison of DCNNs and ACDC networks on two different tasks. Figure~\ref{fig:adcd_regression} shows the evolution of the training loss on a regression task with synthetic data. Figure~\ref{fig:acdc_cifar10} shows the test accuracy on the CIFAR-10 dataset.}
\end{figure*}

{\bf Comparison with ACDC \cite{moczulski2015acdc}.} In Section~\ref{related_work}, we have  discussed the differences between the ACDC framework and our approach from a theoretical perspective. In this section, we conduct experiments to compare the performance of DCNNs with neural networks based on ACDC layers. 
We first reproduce the experimental setting from \cite{moczulski2015acdc}, and compare both approaches using only linear networks (i.e. networks without any ReLU activations). The results are presented in Figure~\ref{fig:adcd_regression}. On this simple setting, both architectures demonstrate good performance, however, DCNNs offer better convergence rate.
In Figure~\ref{fig:acdc_cifar10}, we compare neural networks with ReLU activations on CIFAR-10. The synthetic dataset has been created in order to reproduce the experiment on the regression linear problem proposed by~\cite{moczulski2015acdc}. We draw $X$, $Y$ and $W$ from a uniform distribution between [-1, +1] and $\epsilon$ from a normal distribution with mean 0 and variance $0.01$. The relationship between $X$ and $Y$ is define by $Y = XW + \epsilon$. 

We found that networks which are based only on ACDC layers are difficult to train and offer poor accuracy on CIFAR. (We have tried different initialization schemes including the one from the original paper, and the one we propose in this paper.) Moczulski et al. \cite{moczulski2015acdc} manage to train a large VGG network  however these networks are generally highly redundant, the contribution of the structured layer is difficult to quantify. 
We also observe that adding a single dense layer improves the convergence rate of  ACDC in the linear case networks, which explain the good results of \cite{moczulski2015acdc}. However, it is difficult to characterize the true contribution of the ACDC layers when the network involved a large number of other expressive layers.

In contrast,  deep DCNNs can be trained and offer good performance without additional dense layers  (these results are in line with our experiments on the \yt dataset). We can conclude that DCNNs are able to model complex relations at a low cost.

\begin{figure*}[ht]
\centering
\subfigure[]{
    \hspace{0.5cm}
    \centering
    \begin{tikzpicture}[scale=0.5]
\begin{axis}[
    legend cell align={left},
    xlabel={\large \#weights (x1000) },
    ylabel={Test Accuracy},
    xmin=21, xmax=370,
    ymin=0.2, ymax=0.6,
    legend pos=south east,
    ymajorgrids=true,
    grid style=dashed,
	]
  \addplot[color=red, line width=0.25mm, dashed] table [y=accuracy, x=weights]{data/cifar10/type/dense.dat};
  \addplot[mark=triangle, color=blue, line width=0.4mm] table [y=accuracy, x=weights]{data/cifar10/type/circulant.dat};
  \addplot[mark=square, color=red, line width=0.4mm] table [y=accuracy, x=weights]{data/cifar10/type/diag_toeplitz.dat};
  \addplot[mark=o, color=gray, line width=0.4mm] table [y=accuracy, x=weights]{data/cifar10/type/toeplitz.dat};
  \addplot[mark=diamond, color=brown, line width=0.4mm] table [y=accuracy, x=weights]{data/cifar10/type/low_rank.dat};
    \legend{
      Dense (9M weights),
      DCNN,
      DTNN,
      Toeplitz network,
      Low Rank network, 
     }
\end{axis}
\end{tikzpicture}
    \label{fig:cifar10_type}
    }\hspace{0.9cm}
\subfigure[]{
    \centering
    \begin{tikzpicture}[scale=0.5]
\begin{axis}[
    xlabel={\large \#weights (x1000) },
    ylabel={Test Accuracy},
    legend pos=outer north east,
    legend cell align={left},
    ymajorgrids=true,
    grid style=dashed,
    ]
    \addplot[mark=triangle*,blue] coordinates {(140, 0.7017)}; 
    \addplot[mark=triangle*,red] coordinates {(420, 0.7286)}; 
    \addplot[mark=diamond*,blue] coordinates {(110, 0.7111)}; 
    \addplot[mark=diamond*,red] coordinates {(388, 0.7205)}; 
    \addplot[mark=square*,green] coordinates {(124, 0.757)}; 
    \addplot[mark=square*,blue] coordinates {(217, 0.7856)}; 
    \addplot[mark=square*,red] coordinates {(66, 0.7535)}; 
    \addplot[mark=square*,brown] coordinates {(191, 0.764)}; 
    \legend{
        Scattering + LDR-SD (r=1),
        Scattering + LDR-SD (r=10),
        Scattering + Toeplitz-like (r=1),
        Scattering + Toeplitz-like (r=10),
        Scattering + 1 DC,
        Scattering + 3 DC,
        Scattering Avg pooling + 3 DC,
        Scattering by channel + 4 DC,
    }
\end{axis}
\end{tikzpicture}

		
    \label{fig:cifar10_with_channels_xp}
    \hspace{-1.5cm}
    }
    \caption{Figure~\ref{fig:cifar10_type}: network size vs. accuracy compared on Dense networks, DCNNs (our approach), DTNNs (our approach), neural networks based on Toeplitz matrices and neural networks based on Low Rank-based matrices. DCNNs outperforms alternatives structured approaches. Figure~\ref{fig:cifar10_with_channels_xp} shows the accuracy of different structured architecture given the number of trainable parameters.}
\end{figure*}
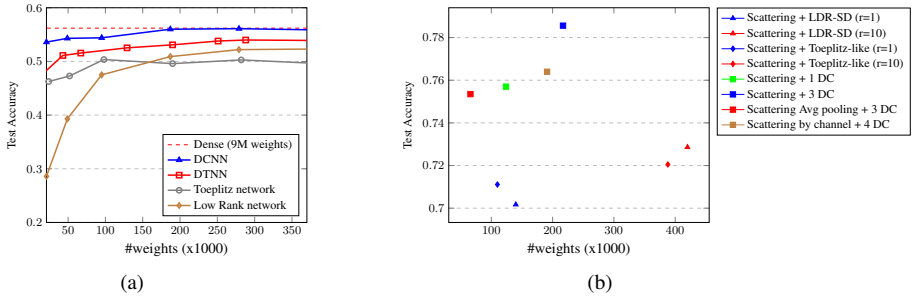

{\bf Comparison with Dense networks, Toeplitz networks and Low Rank networks.}
We now compare DCNNs with other state-of-the-art structured networks by measuring the accuracy on a flattened version of the CIFAR-10 dataset. Our baseline is a dense feed-forward network with a fixed number of weights (9 million weights). We compare with DCNNs and with DTNNs (see below), Toeplitz networks, and Low-Rank networks~\cite{8099498}. We first consider Toeplitz networks which are stacked Toeplitz matrices interleaved with ReLU activations since Toeplitz matrices are closely related to circulant matrices.  Since Toeplitz networks have a different structure (they do not include diagonal matrices), we also experiment using DTNNs, a variant of DCNNs where all the circulant matrices have been replaced by Toeplitz matrices. Finally we conduct experiments using networks based on low-rank matrices as they are also closely related to our work. For each approach, we report the accuracy of several networks with a varying depth ranging from 1 to 40 (DCNNs, Toeplitz networks) and from 1 to 30 (from DTNNs).  For low-rank networks, we used a fixed depth network and increased the rank of each matrix from 7 to 40. We also tried to increase the depth of low rank matrices, but we found that deep low-rank networks are difficult to train so we do not report the results here. We compare all the networks based on the number of weights from 21K (0.2\% of the dense network) to 370K weights (4\% of the dense network) and we report the results in Figure~\ref{fig:cifar10_type}. 
First we can see that the size of the networks correlates positively with their accuracy which demonstrate successful training in all cases. We can also see that the DCNNs achieves the maximum accuracy of 56\% with 20 layers ($\sim$ 200K weights) which as as good as the dense networks with only 2\% of the number of weights. Other approaches also offer good performance but they are not able to reach the accuracy of a dense network.

\begin{table}
  \centering
  \small
  \caption{\small LDR networks compared with DCNNs on a flattend version of CIFAR-10. DCNNs outperform all LDR configurations with fewer weights.$^2$}
  \begin{tabular}{lcc}
    \toprule
    \textbf{Architectures} & \textbf{\#Params} & \textbf{Acc.}  \\
    \midrule
    \textit{Dense} & \textit{9.4M}	& \textit{0.562} \\
    \textbf{\textit{DCNN $(5\ layers)$}} & \textbf{49K}	& \textbf{0.543} \\
    \textbf{\textit{DCNN $(2\ layers)$}} & \textbf{21K} & \textbf{0.536} \\
    LDR--TD	$(r = 2)$	        & 64K	& 0.511 \\
    LDR--TD	$(r = 3)$	        & 70K	& 0.473 \\
    Toeplitz-like $(r=2)$	    & 46K	& 0.483 \\
    Toeplitz-like $(r =3)$	    & 52K    & 0.496 \\
    \bottomrule
    \end{tabular}
    \label{table:xp_ldr}
\end{table}

\begin{table}
  \centering
  \small
  \caption{\small Two depths scattering on CIFAR-10 followed by LDR or DC layer. Networks with DC layers outperform all LDR configurations with fewer weights.}
  \begin{tabular}{lcc}
    \toprule
    \textbf{Architectures} & \textbf{\#Params} & \textbf{Acc.}  \\
    \midrule
    \textbf{DC $(1\ layers)$} & \textbf{124K} & \textbf{0.757} \\
    \textbf{DC $(3\ layers)$} & \textbf{217K} & \textbf{0.785} \\
    \textbf{Ensemble x5 DC $(3\ layers)$} &  \textbf{1.08M} & \textbf{0.811} \\
    LDR-SD $(r=1)$ & 140K & 0.701 \\
    LDR-SD $(r=10)$ & 420K & 0.728 \\
    Toeplitz-like $(r=1)$ & 110K & 0.711 \\
    Toeplitz-like $(r=10)$ & 388K & 0.720 \\
    \bottomrule
    \end{tabular}
    \label{table:xp_ldr_scattering}
\end{table}

\footnotetext[2]{Remark: the numbers may differ from the original experiments by \cite{Thomas_NIPS2018_8119} because we use the original dataset instead of a monochrome version)}

\noindent
{\bf Comparison with LDR networks~\cite{Thomas_NIPS2018_8119}.} We now compare DCNNs with the LDR framework using the network configuration experimented in the original paper: a single LDR structured layer followed by a dense layer.  In the LDR framework, we can change the size of a network by adjusting the rank of the residual matrix, effectively capturing matrices with a structure that is close to a known structure but not exactly (e.g. in the LDR framework, Toeplitz matrices can be encoded with a residual matrix with rank=2, so a matrix that can be encoded with a residual of rank=3 can be seen as Toeplitz-like.).
The results are presented in Table~\ref{table:xp_ldr} and demonstrate that DCNNs outperforms all LDR networks both in terms in size and accuracy. 

{\bf Exploiting image features.} Dense layers and DCNNs are not designed to capture task-specific features such as the translation invariance inherently useful in image classification. We can further improve the accuracy of such general purpose architectures on image classification without dramatically increasing the number of trained parameters by stacking them on top of fixed (i.e. non-trained) transforms such as the scattering transform \cite{mallat2010recursive}. In this section we compare the accuracy of various structured networks, enhanced with the scattering transform, on an image classification task, and run comparative experiments on CIFAR-10. 

Our test architecture consists of 2 depth scattering on the RGB images followed by a batch norm and LDR or DC layer. To vary the number of parameters of Scattering+LDR architecture, we increase the rank of the matrix (stacking several LDR matrices quickly exhausted the memory).
The Figure \ref{fig:cifar10_with_channels_xp} and \ref{table:xp_ldr_scattering} shows the accuracy of these architectures given the number of trainable parameters.

First, we can see that the DCNN architecture very much benefits from the scattering transform and is able to reach a competitive accuracy over 78\%.
We can also see that scattering followed by a DC layer systematically outperforms scattering + LDR or scattering + Toeplitz-like with less parameters.

\subsection{Comparison with other compression based approaches (Q6.2)}


\begin{table}
  \centering
    \caption{Comparison with compression based approaches}
    \small
    \begin{tabular}{lcrc}
    \toprule
    \multicolumn{1}{c}{\textbf{Architecture}} & \multicolumn{1}{c}{\textbf{\#Params}} & \textbf{Error (\%)} \\
    \hline \\
    \textit{LeNet \cite{Lecun98gradient-basedlearning}} & \textit{4 257 674} & \textit{0.61} \\
    \multirow{2}[0]{*}{\textbf{DCNN}} & \textbf{25 620} & \textbf{1.74} \\
          & \textbf{31 764} & \textbf{1.60} \\
    \multirow{2}[0]{*}{HashNet \cite{Chen_Hashing_Trick}} & 46 875 & 2.79 \\
          &  78 125 & 1.99 \\
    \multirow{2}[0]{*}{Dark Knowledge \cite{44873}} & 46 875 & 6.32 \\
          &  78 125 & 2.16 \\
    \bottomrule
    \end{tabular}%
  \label{tab:mnist}%
\end{table}%

We provide a comparison with other compression based approaches such as HashNet \cite{Chen_Hashing_Trick}, Dark Knowledge \cite{44873} and Fast Food Transform (FF) \cite{7410530}. 
Table~\ref{tab:mnist} shows the test error of DCNN against other know compression techniques on the MNIST datasets. We can observe that DCNN outperform easily HashNet \cite{Chen_Hashing_Trick} and Dark Knowledge \cite{44873} with fewer number of parameters. The architecture with Fast Food (FF) \cite{7410530} achieves better performance but with convolutional layers and only $1$ Fast Food Layer as the last Softmax layer.

\subsection{DCNNs for large-scale video classification on the \yt dataset (Q6.3)}

To understand the performance of deep DCNNs on large scale applications, we conducted experiments on the \yt video classification with 3.8 training examples introduced by \cite{abu2016youtube}. Notice that we favour this experiment over ImageNet applications because modern image classification architectures involve a large number of convolutional layers, and compressing convolutional layers is out of our scope. 
Also, as mentioned earlier, testing the performance of DCNN architectures mixed with a large number of expressive layers makes little sense. The \yt includes two datasets describing 8 million labeled videos. Both datasets contain audio and video features for each video. In the first dataset ({\em aggregated}) all audio and video features have been aggregated every 300 frames. The second dataset ({\em full}) contains the descriptors for all the frames. To compare the models we use the GAP metric (Global Average Precision) proposed by~\cite{abu2016youtube}. On the simpler {\em aggregated} dataset we compared off-the-shelf DCNNs with a dense baseline with 5.7M weights.  On the full dataset, we designed three new compact architectures based on the state-of-the-art architecture introduced by \cite{abu2016youtube}. 

\noindent
{\bf Experiments on the {\em aggregated} dataset with DCNNs:}
We compared DCNNs with a dense baseline with 5.7 millions weights. The goal of this experiment is to discover a good trade-off between depth and model accuracy. To compare the models we use the GAP metric (Global Average Precision) following the experimental protocol in~\cite{abu2016youtube}, to compare our experiments. 

Table~\ref{table:youtube_agg_xp} shows the results of our experiments on the {\em aggrgated} \yt dataset in terms of number of weights, compression rate and GAP.
We can see that the compression ratio offered by the circulant  architectures is high. This comes at the cost of a little decrease of GAP measure. The 32 layers DCNN is 46 times smaller than the original model in terms of number of parameters while having a close performance.


\begin{table}
  \centering
  \caption{ \small This table shows the GAP score for the \yt dataset with DCNNs. We can see a large increase in the score with deeper networks.}
  \small
  \begin{tabular}{lccc}
    \toprule
    \textbf{Architecture} & \textbf{\#Weights} &
    \textbf{GAP@20} \\
    \hline \\
    \textit{original} & \textit{5.7M} & \textit{0.773} \\
    4 DC & 25 410  (\textit{\bf 0.44}) & 0.599   \\
    32 DC  & 122 178 \textit{(2.11)} & 0.685   \\
    4 DC + 1 FC & 4.46M \textit{(77)} & \textbf{0.747} \\
  \hline
  \end{tabular}
  \label{table:youtube_agg_xp}
\end{table}

\begin{table}
  \centering
  \caption{ \small This table shows the GAP score for the \yt dataset with different layer represented with our DC decomposition.}
  \small
  \begin{tabular}{lccc}
  \toprule
  \textbf{Architecture} & \textbf{\#Weights} & \textbf{GAP@20} \\
  \hline \\
  \textit{original} & \textit{45M} & \textit{0.846} \\
  DBoF with DC   & 36M (\textit{80}) & 0.838 \\
  FC with DC    & 41M (\textit{91}) & \textbf{0.845} \\
  MoE with DC   & 12M (\textit{\bf 26}) & 0.805 \\
  \hline
  \end{tabular}
  \label{table:youtube_full_xp}
\end{table}

\noindent
{\bf Experiments with DCNNs Deep Bag-of-Frames Architecture:} The Deep Bag-of-Frames architecture can be decomposed into three blocks of layers, as illustrated in Figure~\ref{fig:archi_youtube}.
The first block of layers, composed of the Deep Bag-of-Frames embedding (DBoF), is meant to model an embedding of these frames in order to make a simple representation of each video. A second block of fully connected layers (FC) reduces the dimensionality of the output of the embedding and merges the resulting output with a concatenation operation. Finally, the classification block uses a combination of Mixtures-of-Experts (MoE) ~\cite{716791,45619} and Context Gating~\cite{DBLP:journals/corr/MiechLS17} to calculate the final class probabilities. Table~\ref{table:youtube_full_xp} shows the results in terms of number of weights, size of the model (MB) and GAP on the full dataset, replacing the DBoF block reduces the size of the network without impacting the accuracy. We obtain the best compression ratio by replacing the MoE block with DCNNs (26\%) of the size of the original dataset with a GAP score of 0.805 (95\% of the score obtained with the original architecture). We conclude that DCNN are both theoretically sound and of practical interest in real, large scale applications.

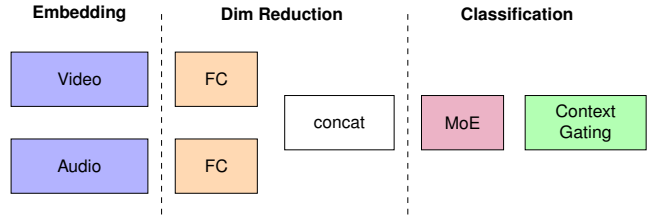
\begin{figure}[ht!]
  \centering
  \scalebox{.72}{\tikzset{%
  >={Latex[width=2mm,length=2mm]},
            base/.style = {rectangle, draw=black, text centered, font=\sffamily},
             box/.style = {base, rounded corners, text depth=3cm, minimum height=4cm, minimum width=3cm},
     transparent/.style = {rectangle, draw=black},
       circulant/.style = {base, fill=yellow!30},
       embedding/.style = {base, fill=blue!30, minimum width=2.5cm, minimum height=1cm},
           other/.style = {base, fill=white!30,  minimum width=2cm, minimum height=1cm},
              fc/.style = {base, fill=orange!30, minimum width=1.5cm, minimum height=1cm},
          gating/.style = {base, fill=green!30, minimum width=2cm, text width=2cm, minimum height=1cm},
             moe/.style = {base, fill=purple!30, minimum width=1.5cm, minimum height=1cm},
}

\begin{tikzpicture}[every node/.style={fill=white, font=\sffamily}, align=center]

  \draw (0.0, +2.)  node [other, draw=none] {\textbf{Embedding}};
  \draw (+3.7, +2.)  node [other, draw=none] {\textbf{Dim Reduction}};
  \draw (+8.0, +2.)  node [other, draw=none] {\textbf{Classification}};

  \draw (0, +0.8)  node [embedding] {Video};
  \draw (0, -0.8)  node [embedding] {Audio};

  \draw (+2.5, +0.8)  node (fc) [fc] {FC};
  \draw (+2.5, -0.8)  node (fc) [fc] {FC};

  \draw (+4.75, 0)  node (fc) [other] {concat};
  \draw (+7.0, 0)  node (moe) [moe] {MoE};
  \draw (+9.25, 0)  node (gating2) [gating] {Context Gating};
 
  \draw (+1.5, +2) [dashed] -- (+1.5, -1.7);
  \draw (+6, +2) [dashed] -- (+6, -1.7);
  
\end{tikzpicture}}
  \caption{This figure shows the state-of-the-art neural network architecture, initially proposed by \cite{abu2016youtube} and later improved by~\cite{DBLP:journals/corr/MiechLS17}, used in our experiment. }
  \label{fig:archi_youtube}
\end{figure}

\paragraph{Architectures \& Hyper-Parameters:} 
For the first set of our experiments (e.g. experiments on CIFAR-10), we train all networks for 200 epochs, a batch size of 200, Leaky ReLU activation with a different slope. We minimize the Cross Entropy Loss with Adam optimizer and use a piecewise constant learning rate of $5 \times 10^{-5}$, $2.5\times10^{-5}$, $5\times10^{-6}$ and $1\times10^{-6}$ after respectively 40K, 60K and 80K steps. For the \yt dataset experiments, we built a neural network based on the SOTA architecture initially proposed by \cite{abu2016youtube} and later improved by~\cite{DBLP:journals/corr/MiechLS17}. Remark that no convolution layer is involved in this application since the input vectors are embeddings of video frames processed using state-of-the-art convolutional neural networks trained on ImageNet. We trained our models with the CrossEntropy loss and used Adam optimizer with a 0.0002 learning rate and a 0.8 exponential decay every 4 million examples. All fully connected layers are composed of 512 units. DBoF, NetVLAD and NetFV are respectively 8192, 64 and 64 of cluster size for video frames and 4096, 32, 32 for audio frames. We used 4 mixtures for the MoE Layer. We used all the available 300 frames for the DBoF embedding. In order to stabilize and accelerate the training, we used batch normalization before each non linear activation and gradient clipping. 

\section{Conclusion}\label{section:conclusion}
This paper deals with the training of diagonal circulant neural networks. To the best of our knowledge, training such networks with a large number of layers had not been done before.
We also endowed this kind of models with theoretical guarantees, hence enriching and refining previous theoretical work from the literature. More importantly, we showed that DCNNs outperform their competing structured alternatives, including the very recent general approach based on LDR networks. Our results suggest that stacking diagonal circulant layers with non linearities improves the convergence rate and the final accuracy of the network. Formally proving these statements constitutes the future directions of this work.  As future work, we would like to generalize the good results of DCNNs to convolutions neural networks. We also believe that circulant matrices deserve a particular attention in deep learning because of their strong ties with convolutions: a circulant matrix operator is equivalent to the convolution operator with circular paddings (as shown in [5]). This fact makes any contribution to the area of circulant matrices particularly relevant to the field of deep learning with impacts beyond the problem of designing compact models. As future work, we would like to generalize our results to deep convolutional neural networks.

\bibliographystyle{ecai}
\bibliography{bibliography}

\begin{thebibliography}{10}

\bibitem{45619}
Sami Abu-El-Haija, Nisarg Kothari, Joonseok Lee, Apostol~(Paul) Natsev, George
  Toderici, Balakrishnan Varadarajan, and Sudheendra Vijayanarasimhan,
  `Youtube-8m: A large-scale video classification benchmark', in {\em
  arXiv:1609.08675}, (2016).

\bibitem{abu2016youtube}
Sami Abu-El-Haija, Nisarg Kothari, Joonseok Lee, Paul Natsev, George Toderici,
  Balakrishnan Varadarajan, and Sudheendra Vijayanarasimhan, `Youtube-8m: A
  large-scale video classification benchmark', {\em arXiv preprint
  arXiv:1609.08675}, (2016).

\bibitem{anca2018eccv}
Alexandre Araujo, Benjamin Negrevergne, Yann Chevaleyre, and Jamal Atif,
  `Training compact deep learning models for video classification using
  circulant matrices', in {\em The 2nd Workshop on YouTube-8M Large-Scale Video
  Understanding at ECCV 2018}, (2018).

\bibitem{Arora19neurips}
Sanjeev Arora, Nadav Cohen, Wei Hu, and Yuping Luo, `Implicit regularization in
  deep matrix factorization', in {\em Neurips}, (05 2019).

\bibitem{ba2014deep}
Jimmy Ba and Rich Caruana, `Do deep nets really need to be deep?', in {\em
  Advances in neural information processing systems}, pp. 2654--2662, (2014).

\bibitem{Chen_Hashing_Trick}
Wenlin Chen, James~T. Wilson, Stephen Tyree, Kilian~Q. Weinberger, and Yixin
  Chen, `Compressing neural networks with the hashing trick', in {\em
  Proceedings of the 32Nd International Conference on International Conference
  on Machine Learning - Volume 37}, ICML'15, pp. 2285--2294. JMLR.org, (2015).

\bibitem{cheng}
Y.~Cheng, F.~X. Yu, R.~S. Feris, S.~Kumar, A.~Choudhary, and S.~F. Chang, `An
  exploration of parameter redundancy in deep networks with circulant
  projections', in {\em 2015 IEEE International Conference on Computer Vision
  (ICCV)}, pp. 2857--2865, (Dec 2015).

\bibitem{47028}
Krzysztof Choromanski, Mark Rowland, Vikas Sindhwani, Richard~E. Turner, and
  Adrian Weller, `Structured evolution with compact architectures for scalable
  policy optimization', in {\em ICML}, (2018).

\bibitem{NIPS2013_5025}
Misha Denil, Babak Shakibi, Laurent Dinh, Marc\textquotesingle~Aurelio Ranzato,
  and Nando de~Freitas, `Predicting parameters in deep learning', in {\em
  Advances in Neural Information Processing Systems 26}, eds., C.~J.~C. Burges,
  L.~Bottou, M.~Welling, Z.~Ghahramani, and K.~Q. Weinberger,  2148--2156,
  Curran Associates, Inc., (2013).

\bibitem{goyal19}
S.~{Goyal}, A.~{Roy Choudhury}, and V.~{Sharma}, `Compression of deep neural
  networks by combining pruning and low rank decomposition', in {\em 2019 IEEE
  International Parallel and Distributed Processing Symposium Workshops
  (IPDPSW)}, pp. 952--958, (2019).

\bibitem{hanin2017universal}
B.~{Hanin}, `{Universal Function Approximation by Deep Neural Nets with Bounded
  Width and ReLU Activations}', {\em ArXiv e-prints}, (August 2017).

\bibitem{hinrichs2011johnson}
Aicke Hinrichs and Jan Vyb{\'\i}ral, `Johnson-lindenstrauss lemma for circulant
  matrices', {\em Random Structures \& Algorithms}, {\bf 39}(3),  391--398,
  (2011).

\bibitem{44873}
Geoffrey Hinton, Oriol Vinyals, and Jeffrey Dean, `Distilling the knowledge in
  a neural network', in {\em NIPS Deep Learning and Representation Learning
  Workshop}, (2015).

\bibitem{Huhtanen2015}
Marko Huhtanen and Allan Per{\"a}m{\"a}ki, `Factoring matrices into the product
  of circulant and diagonal matrices', {\em Journal of Fourier Analysis and
  Applications}, {\bf 21}(5),  1018--1033, (Oct 2015).

\bibitem{716791}
M.~I. Jordan and R.~A. Jacobs, `Hierarchical mixtures of experts and the em
  algorithm', in {\em Proceedings of 1993 International Conference on Neural
  Networks (IJCNN-93-Nagoya, Japan)}, volume~2, pp. 1339--1344 vol.2, (Oct
  1993).

\bibitem{45648}
Jakub Konečný, H.~Brendan McMahan, Felix~X. Yu, Peter Richtarik,
  Ananda~Theertha Suresh, and Dave Bacon, `Federated learning: Strategies for
  improving communication efficiency', in {\em NIPS Workshop on Private
  Multi-Party Machine Learning}, (2016).

\bibitem{Lecun98gradient-basedlearning}
Yann Lecun, Léon Bottou, Yoshua Bengio, and Patrick Haffner, `Gradient-based
  learning applied to document recognition', in {\em Proceedings of the IEEE},
  pp. 2278--2324, (1998).

\bibitem{chong18eccv}
Chong Li and C.~J.~Richard Shi, `Constrained optimization based low-rank
  approximation of deep neural networks', in {\em Computer Vision -- ECCV
  2018}, eds., Vittorio Ferrari, Martial Hebert, Cristian Sminchisescu, and
  Yair Weiss, pp. 746--761, Cham, (2018). Springer International Publishing.

\bibitem{mallat2010recursive}
St{\'e}phane Mallat, `Recursive interferometric representation', in {\em Proc.
  of EUSICO conference, Danemark}, (2010).

\bibitem{DBLP:journals/corr/MiechLS17}
Antoine Miech, Ivan Laptev, and Josef Sivic, `Learnable pooling with context
  gating for video classification', {\em CoRR}, {\bf abs/1706.06905}, (2017).

\bibitem{moczulski2015acdc}
Marcin Moczulski, Misha Denil, Jeremy Appleyard, and Nando de~Freitas, `Acdc: A
  structured efficient linear layer', {\em arXiv preprint arXiv:1511.05946},
  (2015).

\bibitem{muller1998algorithmic}
J{\"o}rn M{\"u}ller-Quade, Harald Aagedal, Th~Beth, and Michael Schmid,
  `Algorithmic design of diffractive optical systems for information
  processing', {\em Physica D: Nonlinear Phenomena}, {\bf 120}(1-2),  196--205,
  (1998).

\bibitem{novikov2015tensorizing}
Alexander Novikov, Dmitrii Podoprikhin, Anton Osokin, and Dmitry~P Vetrov,
  `Tensorizing neural networks', in {\em Advances in Neural Information
  Processing Systems}, pp. 442--450, (2015).

\bibitem{43969}
Tara Sainath and Carolina Parada, `Convolutional neural networks for
  small-footprint keyword spotting', in {\em Interspeech}, (2015).

\bibitem{sanchez1995diagonalizing}
Victoria Sanchez, Pedro Garcia, Antonio~M Peinado, Jos{\'e}~C Segura, and
  Antonio~J Rubio, `Diagonalizing properties of the discrete cosine
  transforms', {\em IEEE transactions on Signal Processing}, {\bf 43}(11),
  2631--2641, (1995).

\bibitem{schmid2000decomposing}
Michael Schmid, Rainer Steinwandt, J{\"o}rn M{\"u}ller-Quade, Martin
  R{\"o}tteler, and Thomas Beth, `Decomposing a matrix into circulant and
  diagonal factors', {\em Linear Algebra and its Applications}, {\bf 306}(1-3),
   131--143, (2000).

\bibitem{Sedghi18iclr}
Hanie Sedghi, Vineet Gupta, and Philip Long, `The singular values of
  convolutional layers', in {\em ICLR}, (2018).

\bibitem{Thomas_NIPS2018_8119}
Anna Thomas, Albert Gu, Tri Dao, Atri Rudra, and Christopher R\'{e}, `Learning
  compressed transforms with low displacement rank', in {\em Advances in Neural
  Information Processing Systems 31}, eds., S.~Bengio, H.~Wallach,
  H.~Larochelle, K.~Grauman, N.~Cesa-Bianchi, and R.~Garnett,  9066--9078,
  Curran Associates, Inc., (2018).

\bibitem{DBLP:conf/iclr/TrabelsiBZSSSMR18}
Chiheb Trabelsi, Olexa Bilaniuk, Ying Zhang, Dmitriy Serdyuk, Sandeep
  Subramanian, Jo{\~{a}}o~Felipe Santos, Soroush Mehri, Negar Rostamzadeh,
  Yoshua Bengio, and Christopher~J. Pal, `Deep complex networks', in {\em 6th
  International Conference on Learning Representations, {ICLR} 2018, Vancouver,
  BC, Canada, April 30 - May 3, 2018, Conference Track Proceedings}, (2018).

\bibitem{7410530}
Z.~Yang, M.~Moczulski, M.~Denil, N.~d.~Freitas, A.~Smola, L.~Song, and Z.~Wang,
  `Deep fried convnets', in {\em 2015 IEEE International Conference on Computer
  Vision (ICCV)}, pp. 1476--1483, (Dec 2015).

\bibitem{8099498}
X.~Yu, T.~Liu, X.~Wang, and D.~Tao, `On compressing deep models by low rank and
  sparse decomposition', in {\em 2017 IEEE Conference on Computer Vision and
  Pattern Recognition (CVPR)}, pp. 67--76, (July 2017).

\bibitem{pmlr-v70-zhao17b}
Liang Zhao, Siyu Liao, Yanzhi Wang, Zhe Li, Jian Tang, and Bo~Yuan,
  `Theoretical properties for neural networks with weight matrices of low
  displacement rank', in {\em Proceedings of the 34th International Conference
  on Machine Learning}, eds., Doina Precup and Yee~Whye Teh, volume~70 of {\em
  Proceedings of Machine Learning Research}, pp. 4082--4090, International
  Convention Centre, Sydney, Australia, (06--11 Aug 2017). PMLR.

\end{thebibliography}

\setcounter{section}{0}
\setcounter{equation}{0}
\setcounter{figure}{0}
\setcounter{defn}{0}
\setcounter{lem}{0}
\setcounter{cor}{0}
\setcounter{property}{0}
\setcounter{thm}{0}

\newpage
\onecolumn

\begin{adjustwidth}{30pt}{30pt}

\section*{\Huge Supplemental Material -- Understanding and Training Deep Diagonal Circulant Neural Networks}

\vspace{1cm}

\section{Notations \& Definition}
We note $\mathfrak{R}(z)$ and $\mathfrak{I}(z)$  the real and imaginary parts the complex number $z$. We note $\left(\cdot\right)_{t}$ is the $t^{th}$ component of a vector. Let $\ci$ be the imaginary number defined by $\ci^2=-1$. 
Define $\mathbf{1}_{n}$ as the \emph{n-}vector of ones. Also, we note $[n]=\{1,\ldots , n\}$. The rectified linear unit on the complex domain is defined by $ReLU(z)=\max\left(0,\mathfrak{R}(z)\right)+\ci\max\left(0,\mathfrak{I}(z)\right)$. The notation $\left|\cdot\right|$ refers to the complex modulus.
Finally, define the \emph{cyclic shift }matrix $S \in \mathbb{R}^{n\times n}$ as follows: 

\[
S=\left[\begin{array}{ccccc}
0 &  &  &  & 1\\
1 & 0\\
 & 1 & \ddots\\
 &  & \ddots & 0\\
 &  &  & 1 & 0
\end{array}\right]
\]
\noindent
We introduce some necessary definitions regarding neural networks.  

\begin{defn}[\drn]\label{drn-appendix}
Given $L$ weight matrices $W = (W_1, \ldots, W_L)$ with $W_i \in \mathbb C^{n\times n}$ and  $L$ bias vectors $b = (b_1, \ldots, b_L)$  with  $b_i \in \mathbb C^n$, a {\em deep \relu network} is a function $f_{W_L, b_L} : \mathbb C^n \rightarrow \mathbb C^n$ such that $f_{W, b}(x) =  (f_{W_L, b_L} \circ \ldots \circ f_{W_1, b_1})(x)$ where $f_{W_i, b_i}(x) = \phi(W_i x + b_i)$ and $\phi(.)$ is a \relu non-linearity
\footnote{Because our networks deal with complex numbers, we use an extension of the \relu function to the complex domain. The most straightforward extension defined in \cite{DBLP:conf/iclr/TrabelsiBZSSSMR18} is as follows: $\mathrm{\relu}(z)=\mathrm{ReLU}\left(\mathfrak{R}(z)\right)+i\mathrm{ReLU}\left(\mathfrak{I}(z)\right)$, where $\mathfrak{R}$ and $\mathfrak{I}$ refer to the real and imaginary parts of $z$.}
In the rest of this paper, we call $L$ and $n$ respectively the depth and the width of the network. Moreover, we call {\em total rank $k$}, the sum of the ranks of the matrices $W_{1}\ldots W_{L}$. i.e. $k = \sum_{i=1}^L rank(W_i)$.
\end{defn}

In the rest of this paper, we call $L$ and $n$ respectively the depth and the width of the network. Moreover, we call {\em total rank $k$}, the sum of the ranks of the matrices $W_{1}\ldots W_{L}$. i.e. $k = \sum_{i=1}^L rank(W_i)$.

\section{Proofs of Section 3}

\begin{thm}
(Reformulation from Huhtanen et al. \cite{Huhtanen2015})\label{thm:huhtanen-appendix}
For any given matrix $A\in\mathbb{C}^{n\times n}$, for any $\epsilon > 0$, there exists a sequence of matrices $B_1 \ldots B_{2n-1}$ where $B_{i}$ is a circulant matrix if $i$ is odd, and a diagonal matrix otherwise, such that $\left\Vert B_{1}B_{2}\ldots B_{2n-1}-A \right\Vert < \epsilon$.
Moreover, if $A$ can be decomposed as $A=\sum_{i=1}^{k}D_{i}S^{i-1}$ where $S$ is the cyclic-shift matrix and $D_{1}\ldots D_{k}$ are diagonal matrices, then $A$ can be written as a product $B_{1}B_{2}\ldots B_{2k-1}$ where $B_{i}$ is a circulant matrix if $i$ is odd, and a diagonal matrix otherwise.
\end{thm}

\begin{thm}(Rank-based circulant decomposition)
\label{prop:rank-decomposition-appendix}Let $A\in\mathbb{C}^{n\times n}$ be
a matrix of rank at most $k$. Assume that $n$ can be divided by $k$. For
any $\epsilon>0$, there exists a sequence of $4k+1$ matrices $B_{1},\ldots,B_{4k+1},$ where $B_{i}$ is a circulant matrix if $i$ is odd, and a diagonal matrix otherwise, such that $\Vert B_1B_2\ldots B_{4k+1} - A\Vert < \epsilon$
\end{thm}

\begin{proof} \emph{(Theorem \ref{prop:rank-decomposition-appendix})}
Let $U\Sigma V^{T}$ be the SVD decomposition of $M$ where $U,V$
and $\Sigma$ are $n\times n$ matrices. Because $M$ is of rank $k$,
the last $n-k$ columns of $U$ and $V$ are null. In the following,
we will first decompose $U$ into a product of matrices $WRO$, where
$R$ and $O$ are respectively circulant and diagonal matrices, and
$W$ is a matrix which will be further decomposed into a product of
diagonal and circulant matrices. Then, we will apply the same decomposition
technique to $V$. Ultimately, we will get a product of $4k+2$ matrices
alternatively diagonal and circulant.

Let $R=circ(r_{1}\ldots r_{n})$. Let $O$ be a $n\times n$ diagonal
matrix where $O_{i,i}=1$ if $i\le k$ and $0$ otherwise. The $k$
first columns of the product $RO$ will be equal to that of $R$,
and the $n-k$ last colomns of $RO$ will be zeros. For example, if
$k=2$, we have:

\[
RO=\left(\begin{array}{ccccc}
r_{1} & r_{n} & 0 & \cdots & 0\\
r_{2} & r_{1}\\
r_{3} & r_{2} & \vdots &  & \vdots\\
\vdots & \vdots\\
r_{n} & r_{n-1} & 0 & \cdots & 0
\end{array}\right)
\]

Let us define $k$ diagonal matrices $D_{i}=diag(d_{i1}\ldots d_{in})$
for $i\in[k]$. For now, the values of $d_{ij}$ are unknown, but
we will show how to compute them. Let $W=\sum_{i=1}^{k}D_{i}S^{i-1}$.
Note that the $n-k$ last columns of the product $WRO$ will be zeros.
For example, with $k=2$, we have:

\[
W=\left[\begin{array}{ccccc}
d_{1,1} &  &  &  & d_{2,1}\\
d_{2,2} & d_{1,2}\\
 & d_{2,3} & \ddots\\
 &  & \ddots\\
 &  &  & d_{2,n} & d_{1,n}
\end{array}\right]
\]

\[
WRO=\left(\begin{array}{ccccc}
r_{1}d_{11}+r_{n}d_{21} & r_{n}d_{11}+r_{n-1}d_{21} & 0 & \cdots & 0\\
r_{2}d_{12}+r_{1}d_{22} & r_{1}d_{12}+r_{n}d_{22}\\
 &  & \vdots &  & \vdots\\
\vdots & \vdots\\
r_{n}d_{1n}+r_{n-1}d_{2n} & r_{n-1}d_{1n}+r_{n-2}d_{2n} & 0 & \cdots & 0
\end{array}\right)
\]
We want to find the values of $d_{ij}$ such that $WRO=U$. We can
formulate this as linear equation system. In case $k=2$, we get:

\[
\left(\begin{array}{cccccccc}
r_{n} & r_{1}\\
r_{n-1} & r_{n}\\
 &  & r_{1} & r_{2}\\
 &  & r_{n} & r_{1}\\
 &  &  &  & r_{2} & r_{3}\\
 &  &  &  & r_{1} & r_{2}\\
 &  &  &  &  &  & \ddots\\
 &  &  &  &  &  &  & \ddots
\end{array}\right)\times\left(\begin{array}{c}
d_{2,1}\\
d_{1,1}\\
d_{2,2}\\
d_{1,2}\\
d_{2,3}\\
d_{1,3}\\
\vdots\\
\vdots
\end{array}\right)=\left(\begin{array}{c}
U_{1,1}\\
U_{1,2}\\
U_{2,1}\\
U_{2,2}\\
\\
\\
\vdots\\
\\
\end{array}\right)
\]
The $i^{th}$ bloc of the bloc-diagonal
matrix is a Toeplitz matrix induced by a subsequence of length $k$ of $(r_1,\ldots r_n,r_1 \ldots r_n)$. Set $r_{j}=1$ for all $j\in\{k,2k,3k,\ldots n\}$
and set $r_{j}=0$ for all other values of $j$. Then it is easy to
see that each bloc is a permutation of the identity matrix. Thus,
all blocs are invertible. This entails that the block diagonal matrix
above is also invertible. So by solving this set of linear equations,
we find $d_{1,1}\ldots d_{k,n}$ such that $WRO=U$. We can apply
the same idea to factorize $V=W'.R.O$ for some matrix $W'$. Finally,
we get
\[
A=U\Sigma V^{T}=WRO\Sigma O^{T}R^{T}W^{'T}
\]

Thanks to Theorem~\ref{thm:huhtanen-appendix}, $W$ and $W'$ can both be
factorized in a product of $2k-1$ circulant and diagonal matrices.
Note that $O\Sigma O^{T}$ is diagonal, because all three are diagonal.
Overall, $A$ can be represented with a product of $4k+2$ matrices,
alternatively diagonal and circulant.
\end{proof}

\section{Proofs of Section 4}

\begin{lem}
\label{lem:product_of_mat_to_DNN}Let $W_{L},\ldots W_{1}\in\mathbb{C}^{n\times n}$,
$b\in\mathbb{C}^{n}$ and let $\mathcal{X}\subset\mathbb{C}^{n}$
be a bounded set. There exists $\beta_{L} \ldots \beta_{1} \in \mathbb{C}^{n}$
such that for all $x \in \mathcal{X}$ we have $f_{W_{L},\beta_{L}}\circ\ldots\circ f_{W_{1},\beta_{1}}(x) = ReLU\left(W_{L}W_{L-1} \ldots W_{1}x+b \right)$.
\end{lem}

\begin{proof} \emph{(Lemma \ref{lem:product_of_mat_to_DNN})}
Define $S=\left\{ \left(\left(\prod_{k=1}^{j}W_{k}\right)x\right)_{t}:x\in\mathcal{X},t\in[n],j\in[L]\right\} $.
Let $\Omega=\max\left\{ \mathfrak{R}(v):v\in S\right\} +\ci \max\left\{ \mathfrak{I}(v):v\in S\right\} $.
Intuitively, the real and imaginary parts of $\Omega$ are the largest any activation in the network can have.
Define $h_{j}(x) = W_{j}x + \beta_{j}$. Let $\beta_{1}=\Omega\mathbf{1}_{n}$. Clearly, for
all $x \in \mathcal{X}$ we have $h_{1}(x)\ge0$, so $ReLU \circ h_{1}(x) = h_{1}(x)$.
More generally, for all $ j < n-1$ define $\beta_{j+1} = \mathbf{1}_{n} \Omega - W_{j+1} \beta_{j}$.
It is easy to see that for all $j < n$ we have $h_{j} \circ \ldots \circ h_{1}(x) = W_{j}W_{j-1} \ldots W_{1}x + \mathbf{1}_{n} \Omega$.
This guarantees that for all $j < n$, $h_{j} \circ \ldots \circ h_{1}(x) = ReLU \circ h_{j} \circ \ldots \circ ReLU \circ h_{1}(x)$.
Finally, define $\beta_{L} = b-A_{L} \beta_{L-1}$. We have, $ReLU \circ h_{L} \circ \ldots \circ ReLU \circ h_{1}(x) = ReLU \left(W_{j} \ldots W_{1}x+b \right)$. 
\end{proof}

\begin{lem}\label{mainth_-appendix}
Let $\mathcal{N}$ be a deep ReLU network of width $n$ and depth $L$, and let $\mathcal{X} \subset \mathbb{C}^{n}$ be a bounded set. For any $\epsilon > 0$, there exists a DCNN $\mathcal{N}'$ of width $n$ and of depth $(2n-1)L$ such that $\Vert \mathcal{N}(x) - \mathcal{N}'(x) \Vert < \epsilon$ for all $x \in \mathcal{X}$.
\end{lem}

\begin{proof} \emph{(Lemma \ref{mainth_-appendix})}
Assume $\mathcal{N}=f_{W_{L},b_{L}} \circ \ldots \circ f_{W_{1},b_{1}}$.
By theorem $\text{\ref{thm:huhtanen-appendix}}$, for any $\epsilon '> 0$,
any matrix $W_{i}$, there exists a sequence of $2n-1$ matrices $C_{i,n} D_{i,n-1} C_{i,n-1}\ldots D_{i,1}C_{i,1}$
such that $\left\Vert\prod_{j=0}^{n-1}D_{i,n-j}C_{i,n-j}-W_{i}\right\Vert<\epsilon'$,
where $D_{i,1}$ is the identity matrix. By lemma \ref{lem:product_of_mat_to_DNN},
we know that there exists $\left\{ \beta_{ij}\right\} _{i\in[L],j\in[n]}$
such that for all $i\in[L]$, $f_{D_{in}C_{in},\beta_{in}} \circ \ldots \circ f_{D_{i1}C_{i1},\beta_{i1}}(x) = ReLU\left(D_{in}C_{in}\ldots C_{i1}x+b_{i}\right)$.

Now if $\epsilon'$ tends to zero, $\left\Vert f_{D_{in}C_{in},\beta_{in}}\circ\ldots\circ f_{D_{i1}C_{i1},\beta_{i1}}-ReLU\left(W_{i}x+b_{i}\right)\right\Vert $
will also tend to zero for any $x\in\mathcal{X}$, because the ReLU
function is continuous and $\mathcal{X}$ is bounded. Let $\mathcal{N}' = f_{D_{1n}C_{1n},\beta_{1n}}\circ\ldots\circ f_{D_{i1}C_{i1},\beta_{i1}}$.
Again, because all functions are continuous, for all $x\in\mathcal{X}$,
$\left\Vert \mathcal{N}(x)-\mathcal{N}'(x)\right\Vert $ tends to
zero as $\epsilon'$ tends to zero.
\end{proof}

\begin{cor}\label{cor:universal-appendix}
Bounded width DCNNs are universal approximators in the following sense: for any continuous function $f:[0,1]^{n}\rightarrow\mathbb{R}_+$ of bounded supremum norm,
for any $\epsilon>0$, there exists a DCNN
$\mathcal{N}_{\epsilon}$ of width $n+3$ such that $\forall x\in[0,1]^{n+3}$, $\left|f(x_{1}\ldots x_{n})-\left(\mathcal{N}_{\epsilon}\left(x\right)\right)_{1}\right|<\epsilon$, where $\left(\cdot\right)_{i}$ represents the $i^{th}$ component of a vector.
\end{cor}

\begin{proof} \emph{(Corollary \ref{cor:universal-appendix})}
It has been shown recently in \cite{hanin2017universal} that for any continuous function $f:[0,1]^{n}\rightarrow\mathbb{R}_+$ of bounded supremum norm, for any $\epsilon>0$, there exists a dense neural network $\mathcal{N}$ with an input layer of width $n$, an output layer of width $1$, hidden layers of width $n+3$ and ReLU activations such that $\forall x\in[0,1]^n, \left|f(x)-\mathcal{N}\left(x\right)\right|<\epsilon$. From $\mathcal{N}$, we can easily build a deep ReLU network $\mathcal{N'}$ of width exactly $n+3$, such that $\forall x\in[0,1]^{n+3}$, $\left|f(x_{1}\ldots x_{n})-\left(\mathcal{N}'\left(x\right)\right)_{1}\right|<\epsilon$. Thanks to lemma \ref{mainth_-appendix}, this last network can be approximated arbitrarily well by a DCNN of width $n+3$.
\end{proof}

\begin{thm}(Rank-based expressive power of diagonal circulant neural networks)\\
\label{prop:low_rank_nn-appendix}
\noindent Let $\mathcal{N}:f_{W_{L},b_{L}}\circ\ldots\circ f_{W_{1},b_{1}}$ be a deep ReLU network of width $n$, depth $L$ and a total rank $k$. Assume $n$ is a power of $2$. Let $\mathcal{X} \subset \mathbb{C}^{n}$ be a bounded set. For any $\epsilon>0$, there exists a DCNN $\mathcal{N}'$ of width $n$ such that $\left\Vert \mathcal{N}(x)-\mathcal{N}'(x)\right\Vert <\epsilon$ for all $x\in\mathcal{X}$. In addition, the depth of $\mathcal{N}'$ is bounded by $9k$. Moreover, if the rank of each matrix $A_i$ divides $n$, then the depth of $\mathcal{N}'$ is bounded by $L+4k$.
\end{thm}

\begin{proof} \emph{(Theorem \ref{prop:low_rank_nn-appendix})}
Let $k_{1}\ldots k_{L}$ be the ranks of matrices $W_{1}\ldots W_{L}$,
which are $n$-by-$n$ matrices. For all $i$, there exists $k_{i}'\in\{k_{i}\ldots2k_{i}\}$
such that $k'_{i}$ is a power of $2$. Due to the fact that $n$
is also a power of $2$, $k'_{i}$ divides $n$. By theorem \ref{prop:rank-decomposition-appendix},
for all $i$ each matrix $W_{i}$ can be decomposed as an alternating
product of diagonal-circulant matrices $B_{i,1}\ldots B_{i,4k'_{i}+1}$
such that $\left\Vert W_{i}-B_{i,1}\times\ldots\times B_{i,4k'_{i}+1}\right\Vert <\epsilon$.
Using the exact same technique as in lemma \ref{mainth_-appendix}, we can
build a DCNN $\mathcal{N}'$ using matrices $B_{1,1}\ldots B_{L,4k'_{L}+1}$,
such that $\left\Vert \mathcal{N}(x)-\mathcal{N}'(x)\right\Vert <\epsilon$
for all $x\in\mathcal{X}$. The total number of layers is $\sum_{i}\left(4k_{i}'+1\right)\le L+8\sum_{i}k_{i}\le L+8.\textrm{total rank}\le9.\textrm{total rank}$.

\end{proof}

Finally, what if we choose to use small depth networks to approximate deep ReLU networks where matrices are not of low rank? To answer this question, we first need to show the negative impact of replacing matrices by their low rank approximators in neural networks:

\begin{prop}
\label{prop:relu_to_svd}
Let $\mathcal{N} = f_{W_{L},b_{L}} \circ \ldots \circ f_{W_{1},b_{1}}$ be a Deep ReLU network, where $W_{i} \in \mathbb{C}^{n \times  n}, b_{i} \in \mathbb{C}^{n}$ for all $i \in [L]$. Let $\tilde{W}_{i}$ be the matrix obtained by an SVD approximation of rank $k$ of matrix $W_{i}$. Let $\sigma_{i,j}$ be the $j^{th}$ singular value of $W_{i}$. Define $\tilde{\mathcal{N}} = f_{\tilde{W_{L}},b_{L}} \circ \ldots \circ f_{\tilde{W_{1}},b_{1}}$. Then, for any $x \in \mathbb{C}^{n}$, we have:
\[
\left \Vert \mathcal{N}\left(x\right)-\tilde{\mathcal{N}}\left(x\right)\right\Vert \le\frac{\left(\sigma_{max,1}^{L}-1\right)R\sigma_{max,k}}{\sigma_{max,1}-1}
\]
\noindent
where $R$ is an upper bound on norm of the output of any layer in $\mathcal{N}$, and $\sigma_{max,j}=\max_{i}\sigma_{i,j}$.
\end{prop}

\begin{proof} \emph{(Proposition \ref{prop:relu_to_svd})}
Let $x_{0}\in\mathbb{C}^{n}$ and $\tilde{x}_{0}=x_{0}$. For all $i\in[L]$,
define $x_{i}=ReLU\left(W_{i}x_{i-1}+b\right)$ and $\tilde{x}_{i}=ReLU\left(\tilde{W_{i}}\tilde{x}_{i-1}+b\right)$.
By lemma \ref{lem:bound_one_layer}, we have 
\[
\left\Vert x_{i}-\tilde{x}_{i}\right\Vert \le\sigma_{i,k+1}\left\Vert x_{i-1}\right\Vert +\sigma_{i,1}\left\Vert x_{i-1}-\tilde{x}_{i-1}\right\Vert 
\]
Observe that for any sequence $a_{0},a_{1}\ldots$ defined recurrently
by $a_{0}=0$ and $a_{i}=ra_{i-1}+s$, the recurrence relation can
be unfold as follows: $a_{i}=\frac{s\left(r^{i}-1\right)}{r-1}$.
We can apply this formula to bound our error as follows:
\[
\left\Vert x_{l}-\tilde{x}_{l}\right\Vert \le\frac{\left(\sigma_{max,1}^{l}-1\right)\sigma_{max,k}\max_{i}\left\Vert x_{i}\right\Vert }{\sigma_{max,1}-1}
\]
\end{proof}

\begin{lem}
\label{lem:bound_one_layer}Let $W\in\mathbb{C}^{n\times n}$ with
singular values $\sigma_{1}\ldots\sigma_{n}$, and let $x,\tilde{x}\in\mathbb{C}^{n}$.
Let $\tilde{W}$ be the matrix obtained by a SVD approximation of rank
$k$ of matrix $W$. Then we have:

\[
\left\Vert ReLU\left(Wx+b\right)-ReLU\left(\tilde{W}\tilde{x}+b\right)\right\Vert \le\sigma_{k+1}\left\Vert x\right\Vert +\sigma_{1}\left\Vert \tilde{x}-x\right\Vert 
\]
\end{lem}

\begin{proof} {\em (Lemma~\ref{lem:bound_one_layer})}
Recall that $\left\Vert W\right\Vert _{2}=\sup_{z}\frac{\left\Vert Wz\right\Vert_2 }{\left\Vert z\right\Vert_2 }=\sigma_{1}=\left\Vert \tilde{W}\right\Vert _{2}$,
because $\sigma_{1}$ is the greatest singular value of both $W$
and $\tilde{W}$. Also, note that $\left\Vert W-\tilde{W}\right\Vert _{2}=\sigma_{k+1}$.
Let us bound the formula without ReLUs:

\begin{align*}
\left\Vert \left(Wx+b\right)-\left(\tilde{W}\tilde{x}+b\right)\right\Vert  & =\left\Vert \left(Wx+b\right)-\left(\tilde{W}\tilde{x}+b\right)\right\Vert \\
 & =\left\Vert Wx-\tilde{W}x-\tilde{W}\left(\tilde{x}-x\right)\right\Vert \\
 & \le\left\Vert \left(W-\tilde{W}\right)x\right\Vert +\left\Vert \tilde{W}\right\Vert _{2}\left\Vert \tilde{x}-x\right\Vert \\
 & \le\left\Vert x\right\Vert \sigma_{k+1}+\sigma_{1}\left\Vert \tilde{x}-x\right\Vert 
\end{align*}

Finally, it is easy to see that for any pair of vectors $a,b\in\mathbb{C}^{n}$,
we have $\left\Vert ReLU(a)-ReLU(b)\right\Vert \le\left\Vert a-b\right\Vert $.
This concludes the proof.
\end{proof}

\begin{cor}
\label{cor:relu_to_circ}Consider any deep ReLU network $\mathcal{N} = f_{W_{L},b_{L}} \circ \ldots \circ f_{W_{1},b_{1}}$
of depth $L$ and width $n$. Let $\sigma_{max,j} = \max_{i} \sigma_{i,j}$
where $\sigma_{i,j}$ is the $j^{th}$ singular value of $W_{i}$.
Let $\mathcal{X} \subset \mathbb{C}^{n}$ be a bounded set. Let $k$ be an integer dividing $n$. There exists a DCNN $\mathcal{N}' = f_{D_{m}C_{m},b'_{m}} \circ \ldots \circ f_{D_{1}C_{1},b'_{1}}$
of width $n$ and of depth $m=L(4k+1)$, such that for any $x\in\mathcal{X}$:
\[
\left\Vert \mathcal{N}\left(x\right)-\mathcal{N}'\left(x\right)\right\Vert <\frac{\left(\sigma_{max,1}^{L}-1\right)R\sigma_{max,k}}{\sigma_{max,1}-1}
\]
\noindent
where $R$ is an upper bound on the norm of the outputs of each layer
in $\mathcal{N}$.
\end{cor}

\begin{proof} \emph{(Corollary \ref{cor:relu_to_circ})}
Let $\tilde{\mathcal{N}}=f_{\tilde{W_{L}},b_{L}}\circ\ldots\circ f_{\tilde{W_{1}},b_{1}}$, where each $\tilde{W}_{i}$ is the matrix obtained by an SVD approximation of rank $k$ of matrix $W_{i}$. 
With Proposition~\ref{prop:relu_to_svd}, we have an error bound on $\Vert \mathcal{N}\left(x\right)-\tilde{\mathcal{N}}\left(x\right)\Vert $. Now each matrix $\tilde{W}_{i}$ can be replaced by a product of $k$ diagonal-circulant matrices. By theorem \ref{prop:low_rank_nn-appendix}, this product yields a DCNN of depth $m = L(4k+1)$, strictly equivalent to $\tilde{\mathcal{N}}$ on $\mathcal{X}$. The result follows.
\end{proof}

\end{adjustwidth}

\end{document}